\documentclass{article}

\makeatletter
\renewcommand{\maketitle}{%
  \begingroup
    \renewcommand{\thefootnote}{\fnsymbol{footnote}}%
    \if@twocolumn
      \@topnewpage[\@maketitle]
    \else
      \newpage
      \global\@topnum\z@
      \@maketitle
    \fi
    \thispagestyle{empty}
    \@thanks
  \endgroup
  \setcounter{footnote}{0}%
}

\renewcommand{\@maketitle}{%
  \newpage
  \null
  \vskip 2em%
  \begin{center}%
    {\LARGE \@title \par}%
    \vskip 1em%
    {\large \lineskip .5em%
      \begin{tabular}[t]{c}%
        \@author
      \end{tabular}\par}%
    \vskip 1em%
    {\large \@date}%
  \end{center}%
  \vskip 2em%
}
\makeatother

\usepackage[utf8]{inputenc}
\usepackage[T1]{fontenc} 
\usepackage{xcolor}

\definecolor{midblue}{RGB}{0, 70, 140}

\usepackage{hyperref}
\hypersetup{
    colorlinks=true, 
    linkcolor=darkblue,
    citecolor=darkblue,
    urlcolor=darkblue,
    allcolors=midblue
}

\usepackage{arxiv}
\usepackage{url} 
\usepackage{booktabs} 
\usepackage{amsfonts} 
\usepackage{nicefrac} 
\usepackage{amsmath} 
\usepackage{microtype} 
\usepackage{cleveref}  
\usepackage{graphicx}
\usepackage{natbib}
\usepackage{doi}
\usepackage{threeparttable}

\usepackage{amssymb}
\usepackage{amsthm}

\usepackage{subcaption}

\usepackage{setspace}
\usepackage{tabularx}

\makeatletter
\providecommand{\floatnote}[1]{%
  \par\smallskip
  {\footnotesize\noindent\textit{Note:} #1\par}%
}
\makeatother

\usepackage{fancyhdr}
\usepackage[explicit]{titlesec}
\titleformat{\section}{\large\bfseries}{\thesection}{1em}{#1}

\usepackage{titling}

\usepackage{amsmath}
\usepackage{graphicx}
\usepackage{rotating}

\usepackage{pdflscape}
\definecolor{darkblue}{RGB}{0, 0, 100}
\definecolor{cbblue}{HTML}{0072B2}
\definecolor{cborange}{HTML}{E69F00}
\definecolor{midnightblue}{RGB}{25,25,112}

\usepackage{url}
\usepackage{booktabs}
\usepackage{amsfonts}
\usepackage{nicefrac}
\usepackage{amsmath}
\usepackage{microtype}

\usepackage{graphicx}
\usepackage{natbib}
\usepackage{threeparttable}

\usepackage{subcaption}

\usepackage{tabularx}

\usepackage[explicit]{titlesec}
\titleformat{\section}{\large\bfseries}{\thesection}{1em}{#1}

\usepackage{mathpazo}

\numberwithin{equation}{section}

\newtheorem{proposition}{Proposition}
\newtheorem{lemma}[proposition]{Lemma}

\usepackage{pgfplots}
\pgfplotsset{compat=1.18} 

\newif\ifcolorfigs
\colorfigstrue

\pretitle{\begin{center}\LARGE\bfseries}
\posttitle{\par\end{center}\vskip 0.5em}
\preauthor{\begin{center}\large}
\postauthor{\end{center}}
\predate{\begin{center}\small}
\postdate{\end{center}}

\usepackage{lineno}
\begin{document}

\title{AI Governance under Political Turnover: The Alignment Surface of Compliance Design}

\author{\href{https://orcid.org/0000-0002-0811-3515}{\includegraphics[scale=0.03]{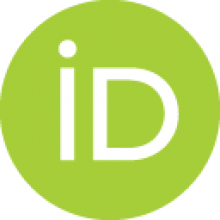}} \hspace{1mm} Andrew J. Peterson
 \thanks{Assistant Professor (Ma\^{i}tre de conf\'{e}rences), University of Poitiers. \href{mailto:andrew.peterson@univ-poitiers.fr}{andrew.peterson@univ-poitiers.fr}}}

\renewcommand{\shorttitle}{Government AI under Political Turnover}

\hypersetup{
pdftitle={AI Governance under Political Turnover: The Alignment Surface of Compliance Design}
pdfsubject={AI governance},
pdfauthor={Andrew J.~Peterson},
pdfkeywords={Artificial Intelligence, democratic fragility},
}

\date{\today}

\maketitle

\bigskip

\begin{abstract}
Governments are increasingly interested in using AI to make administrative decisions cheaper, more scalable, and more consistent. 
But for probabilistic AI to be incorporated into public administration it must be embedded in a compliance layer that makes decisions reviewable, repeatable, and legally defensible. 
That layer can improve oversight by making departures from law easier to detect. 
But it can also create a stable approval boundary that political successors learn to navigate while preserving the appearance of lawful administration. 
We develop a formal model in which institutions choose the scale of automation, the degree of codification, and safeguards on iterative use. 
The model shows when these systems become vulnerable to strategic use from within government, why reforms that initially improve oversight can later increase that vulnerability, and why expansions in AI use may be difficult to unwind. 
Making AI usable can thus make procedures easier for future governments to learn and exploit.
\end{abstract}

\section{Introduction}
\label{sec:intro}

Democracies govern through delegation, but delegation is always shadowed by the problem of control.
Executives, legislatures, and courts cannot directly supervise every administrative act, so democratic government depends on bureaucratic procedures such as rules, documentation, review, and contestation that make delegated action legible and challengeable \citep{mccubbins1987administrative,mccubbins1989structure,scott1998seeing,porter1995trust,power1997audit}.
That challenge is becoming sharper as many welfare states struggle under budget constraints, debt burdens, and frustration with bureaucratic red tape.
Because public services remain labor-intensive, cost growth is hard to avoid, which heightens political pressure for efficiency and increases electoral vulnerability to anti-establishment challengers.

Generative AI appears to offer governments a way to respond to these pressures. 
Large language models can interpret and produce text across cases, draft and personalize official communications, summarize records, translate policy into operational instructions, and even generate code to implement procedures. 
Unlike earlier administrative systems that were narrower or more deterministic, however, these models are probabilistic and can fail in opaque ways. 
For that reason, they are usable in rule-bound administration only insofar as their outputs can be checked against constitutions, statutes, and administrative procedure, and only insofar as those checks generate records that make departures traceable and contestable \citep{kroll2017accountable}. 
Governments have already begun to build such compliance layers through instruments such as impact assessments, transparency registers, logging requirements, and review procedures for public-sector AI \citep{tbs2021directiveadm,ukgds2025atrs,netherlandsalgoregister,eu2024aiact}.
The common aim is to make AI-assisted decisions legible within ordinary administrative oversight. 
This complicates a familiar logic of bureaucratic control, in which rules and procedures reduce drift by making violations visible and sanctionable \citep{huber2002deliberate}.
The same devices that make delegated action more legible to overseers can, after turnover, provide successor executives with a more usable pathway for steering administrative action without setting off tripwires.

We argue that the same reforms that make AI usable in public administration can also create a vulnerability after political turnover.
Contemporary democratic erosion often proceeds through formally lawful channels and avoids direct rupture with the constitutional order even as it weakens effective constraints \citep{bermeo2016backsliding,varol2015stealth,scheppele2018autocratic}.
When official decisions are evaluated through stable criteria, repeatable checks, and inherited administrative workflows, later political actors may be able to study and strategically probe those boundaries from within the state.
This reverses a familiar logic of bureaucratic control in which rules and procedures reduce drift by making violations visible and sanctionable \citep{huber2002deliberate}.
The central tension is that reforms meant to make administration more legible and accountable can also make strategic manipulation easier after power changes hands.

The mechanism is learnability.
Because legal and constitutional constraints are incomplete and contested, operational compliance must translate broad texts into implementable criteria \citep{kaplow1992rules,sunstein1995problems}.
To be defensible to overseers, that translation must yield stable judgments and reviewable records.
Those same features can create a learnable boundary between actions that clear review and actions that trigger concern.
We call that inherited pass-or-flag boundary the \emph{alignment surface}.\footnote{The term does not refer to alignment in the broad AI-safety sense, but to the practical review boundary generated by rules, tests, records, and approval criteria through which AI-mediated decisions are checked to ensure they are in line with law and procedure.} 
In institutional terms, the alignment surface is an inherited review technology: it is built to discipline current administration, but it may later serve successor principals as a reusable map of what passes and what triggers scrutiny.

We do not mean to claim that AI alone can ever generate a learnable administrative boundary. 
Standardized procedures of many kinds can do so. 
What is distinctive about probabilistic and generative AI is the combination of breadth and operationalization pressure. 
These systems can be deployed across heterogeneous administrative tasks, but their outputs are uncertain, non-categorical, and often difficult to defend on their own terms. 
To use them in rule-bound administration, states must therefore fit them within a score-to-action layer that translates model outputs into thresholds, exception rules, reason-giving requirements, and review artifacts. 
That layer is what makes AI administratively usable, but it also makes the approval boundary more stable, portable, and reusable across cases and domains. 
This institutional adaptation, undertaken to make AI governable, can give later insiders a more learnable compliance architecture to probe after political turnover.

We develop a formal model in which the design of oversight institutions is a political choice whose consequences become visible only after turnover occurs.
The baseline model has three stages.
First, safeguards are set upstream by statute, organizational design, and long-run investments.
Second, institutions adopt AI-assisted administrative procedures under modernization pressure.
Third, a successor executive inherits those institutions and may attempt to undermine them either directly or covertly.
The key design choices are how broadly the governed stack is deployed, how standardized and reusable the evaluation criteria become, and what access, review, and remedy institutions exist to constrain iterative probing.
Section~\ref{sec:results} then adds a short post-crisis repair extension that asks whether a temporary pressure episode is later unwound once ordinary conditions return.

The paper does not argue against auditability, standardization, or AI-mediated compliance as such. 
Its institutional claim is narrower: safeguards should be evaluated not only for how well they expose present-day violations, but also for how much they leave behind a reusable pathway for future insiders. 
Section~\ref{sec:discussion} returns to the design implications.

The model yields several institutional results. 
It identifies a threshold at which an inherited compliance architecture becomes politically exploitable after turnover. 
It then shows a codification dilemma: once blatant violations are deterred, reforms that jointly increase auditability and internal standardization can raise risk by making the system easier for bad-faith insiders to learn. 
Finally, it shows why modernization pressure can have lasting effects, because later democratic governments may inherit workflows that are difficult to unwind sufficiently to move back below that threshold.

We don't mean to suggest that governments should simply reject standardization or AI-mediated compliance outright.  
Procedural regularity and auditability are often necessary for legitimacy and for constraining discretion \citep{tyler1990obey,scott1998seeing}.
But we highlight a key democratic risk in AI-mediated bureaucratic governance: auditability may not suffice for democratic robustness. 
A compliance regime can be transparent yet dangerously legible for those seeking to do harm.
Designing institutions to be robust democratically requires preserving oversight visibility while reducing insider learnability.
Where leaders can develop strong priors about where codified bureaucracies can be exploited, turnover can convert the accountability infrastructure into a learnable boundary that enables procedurally clean erosion.

To make this mechanism concrete, consider a stylized case in which an executive seeks to quietly weaponize a regulatory agency against political opponents while preserving the appearance of legal compliance.
In a traditional bureaucracy, doing so at scale requires persuading many human officers to interpret rules selectively and to produce consistent justifications across cases.
That dependence on dispersed human judgment introduces friction, heterogeneity, and leak risk.
If the agency has instead modernized around a language-model-based review process, the successor inherits a far more obedient and legible compliance agent.
The system can process large volumes of cases, generate plausible justifications, and let officials revise submissions until they satisfy the criteria used in review.
What makes the architecture efficient for ordinary governance can therefore also make quiet, scalable within-form erosion easier after turnover.

The remainder of the paper proceeds as follows.
Section~\ref{sec:related} situates the argument in literatures on bureaucratic control, democratic backsliding and autocratic legalism, and algorithmic governance.
Section~\ref{sec:model} presents the model of adoption under turnover and formalizes overt versus within-form channels of democratic failure. 
Section~\ref{sec:results} derives the threshold condition under which the alignment surface becomes exploitable, the codification flip, and a pressure-and-persistence result showing how temporary modernization can create durable inherited vulnerability.
Section~\ref{sec:discussion} interprets the results in institutional terms and develops design implications and empirical signatures for AI-governed compliance systems.
Section~\ref{sec:conclusion} concludes.

\section{Related Literature}
\label{sec:related}


Our argument builds on work on political control of bureaucracy and the design of administrative procedures. 
Under political turnover, procedures chosen by one coalition become the inherited operating environment for successors with different objectives. 
Our argument therefore builds on scholarship on the political control of bureaucracy and the design of administrative procedures as instruments of commitment. 
In principal-agent accounts, elected principals choose statutes and procedures to reduce agency loss when monitoring is costly. 
Seminal work emphasizes that \emph{ex ante} process design can substitute for \emph{ex post} surveillance by structuring how decisions are made and how deviations are detected \citep{mccubbins1987administrative,mccubbins1989structure}. 
Related research shows that legislatures vary administrative procedures to balance political control against expertise and policy flexibility \citep{bawn1995political,epstein1999delegating,weingast1983bureaucratic}. 
We situate our argument in this tradition while taking seriously political turnover, where procedures chosen by one coalition become the inherited operating environment for a successor with different objectives. 
Seen this way, codification can simultaneously limit bureaucratic drift and reduce the informational cost of locating legally defensible routes to consolidation for a hostile successor. 
This framing aligns with work on endogenous policy expertise, but focuses on leaders' strategic search over inherited institutional constraints rather than on bureaucrats' information acquisition \citep{gailmard2007slackers}.

This tradition also connects to work on oversight regimes and the institutional design of monitoring.
Classic accounts distinguish direct ``police patrol'' oversight from ``fire alarm'' oversight that relies on third parties to detect problems and trigger intervention \citep{mccubbins1984oversight,lupia1994learning}.
A central implication is that oversight depends on information and verifiability, since principals can discipline agents only when deviations can be detected, documented, and credibly acted upon \citep{kiewiet1991logic,aberbach2001keeping}.
Administrative procedures therefore operate not only as constraints on agents, but as an informational infrastructure that determines who can observe violations, what counts as admissible proof, and how quickly remedies can be initiated. 
When complaint, appeal, and review mechanisms are standardized and documentation is routinized, fire-alarm oversight lowers verification costs for outsiders \citep{mccubbins1984oversight,lupia1994learning}. 

Bureaucratic discretion is another key dimension. 
\citet{lipsky1980street} emphasizes that street-level officials effectively make policy through discretionary judgments under resource constraints. 
Discretion is structured by organizational routines, monitoring, and incentives, and that formal rules often reshape discretion rather than eliminate it \citep{wilson1989bureaucracy,brehm1997working,maynardmoody2003cops}.
As a result, reform agendas that seek uniformity and procedural regularity trade off two kinds of accountability: they can limit idiosyncratic variation and make decisions easier to justify, but they can also relocate judgment into upstream rule design, data definitions, and the administration of exceptions.
This explains our focus on standardization as an institutional choice that can strengthen oversight while also changing what is legible, predictable, and contestable.

Our emphasis on standardization also connects to work on legibility, quantification, and the unintended consequences of performance measurement.
Scott argues that modern administration depends on legibility, the simplification of complex social realities into standardized categories that make governance possible at scale \citep{scott1998seeing}.
Research on quantification and auditing shows that demands for objectivity often produce systems that are easier to verify formally than to evaluate substantively \citep{porter1995trust,power1997audit}.
A central implication is reactivity, since public measures change incentives and induce adaptation to what is counted and compared \citep{espeland2007rankings}.
In public management, target regimes and metric-based accountability often produce goal displacement and strategic compliance, including gaming that satisfies the evaluation rule while defeating its purpose \citep{bevan2006measured,hood2006gaming,strathern1997improving}.
Taken together, these literatures show that administrative reforms designed to increase legibility and reviewability can also make the criteria of acceptable action more explicit and predictable.

In the literature on democratic backsliding, our mechanism provides a micro-level account of how leaders erode checks while preserving legal form.
Work on autocratic legalism, abusive constitutionalism, and constitutional retrogression emphasizes that contemporary authoritarian consolidation often proceeds through formally lawful changes and strategic use of existing rules \citep{scheppele2018autocratic, corrales2015authoritarian, huq2018lose, landau2013abusive}.
Comparative research likewise highlights executive aggrandizement and other incremental tactics that avoid the overt signals of a coup \citep{bermeo2016backsliding, levitsky2019democracies, varol2015stealth}.
This scholarship richly documents the repertoire of legal maneuvers, but it less often theorizes how procedural precision and administrative standardization shape the ease with which tactics can be identified, justified, and replicated inside the state.
Our contribution is to foreground this institutional margin by treating within-form abuse as a problem of locating and scaling legally plausible actions within an inherited compliance architecture.

This emphasis also complements theories of gradual institutional change that focus on drift, conversion, and reinterpretation when enforcement and contestation are uneven \citep{mahoney2010institutional_change}.
Where this literature stresses how ambiguity creates room for repurposing, we highlight a parallel channel in which codified procedures can make repurposing more operational by reducing interpretive variance across offices and time.
Legal scholarship on rules versus standards makes this duality explicit, since more precise rules can limit discretionary drift while also enabling strategic behavior by actors who can anticipate how compliance will be evaluated \citep{kaplow1992rules,diver1983optimal,schauer1991playing}. 
A similar point appears in accounts of ``constitutional hardball'', which emphasize that formally lawful moves can nonetheless undermine the norms that sustain democratic competition \citep{tushnet2004hardball}. 
Taken together, these literatures suggest that rules can both constrain and enable strategic behavior, but they leave open how procedural precision and administrative standardization shape the cost of identifying and scaling within-form consolidation under turnover. 
We take up that question by focusing on how standardized, machine-operational compliance routines change what is predictable and repeatable for a successor operating inside existing legal constraints.

These dynamics also connect to a growing literature on algorithmic governance and system-level bureaucracy.
\citet{bovens2002system} and \citet{yeung2018algorithmic} emphasize that automated decision systems shift governance toward rule-based, machine-operational procedures whose legitimacy depends on their explainability, auditability, and contestability.
Work on data-driven administration highlights both the welfare-state consequences of automated classification and the new accountability demands created by black-boxed systems \citep{eubanks2018automating,pasquale2015black}.
Legal scholarship similarly emphasizes due process and transparency concerns in automated decision-making and the institutional requirements for accountable systems \citep{citron2008due,kroll2017accountable}.
We build on these debates but emphasize a political risk that is easy to understate when the focus is only on normal-time error or bias: standardized compliance can be appropriated and operationalized by a successor who seeks procedurally clean consolidation.

Recent work suggests that generative AI is shifting ``algorithmic governance'' from discrete classification tasks toward language-mediated workflow integration: drafting and summarizing case records, generating routinized justifications, and supporting frontline and back-office processing at scale \citep{LorenzPersetBerryhill2023GenAI,OECD2024GoverningWithAI}.
Early evidence indicates that uptake is often bottom-up and unevenly governed inside public organizations, prompting rapid proliferation of internal guidance and risk-management practices rather than stable statutory settlement \citep{BrightEtAl2024GenAIWidespread,Weerts2025GenAIinPA,UKHMG2024GenAIFramework}.
As these systems are connected to administrative systems of record, they increasingly resemble ``agentic'' tools that can plan and execute chained actions within institutional constraints, raising oversight problems that require continuous monitoring, operational visibility, and bounded authority \citep{OECD2026AgenticLandscape,IMDA2026MGFAgenticAI,SchmitzRystromBatzner2025AgentOversight}.
For our purposes, the key implication is institutional: embedding generative or agentic AI in governance typically entails formalizing inputs, exceptions, and review protocols, thereby deepening procedural codification and expanding the set of predictable, repeatable moves available to actors who inherit the compliance architecture under turnover.

Finally, we draw on concepts from the technical literature on AI safety to clarify an analogous vulnerability in institutional design. 
In machine learning systems, ``specification gaming'' and ``reward hacking'' describe cases where an optimizing agent attains the formal metric while violating the designer's substantive intent \citep{amodei2016concrete,pan2022effects,skalse2022defining}. 
Related work emphasizes that written objectives are partial signals about underlying aims and can be misread when contexts shift \citep{hadfieldmenell2017inverse,russell2019human}.
A similar logic applies to administrative governance when compliance is defined through standardized, machine-operational procedures that can be satisfied strategically without honoring the normative purpose of the rules \citep{bovens2002system,yeung2018algorithmic}. 
AI-enabled administration may improve logging and traceability, but these records do not by themselves prevent actors from pursuing formally compliant paths that remain substantively abusive. 
This perspective reframes digitized compliance not only as an accountability tool, but also as a standardized interface that can be probed and exploited by political actors with operational access. 
The broader implication is that democratic safeguards should anticipate adversarial adaptation in the same way that safety engineering anticipates failures in complex systems \citep{perrow1984normal}. 

\section{Model}
\label{sec:model}



We model how institutions respond to pressure for more efficient administration when the administrative process they build may later be inherited by a successor with different aims.
Standardized oversight procedures, including rules, documentation, and logs, can deter overt violations by making departures easier to detect and unwind.
But the same standardization can also make the review process predictable enough for insiders to study and exploit, especially where legal constraints are contested and compliance must be operationalized through stable criteria.

The model has three stages.
At $t=0$, slow-moving safeguards are set by statute, organizational design, and long-run investment.
At $t=1$, modernization pressure is realized and institutions choose an adoption architecture, taking those safeguards as given.
At $t=2$, a successor executive inherits and uses the resulting administrative process.
The successor is democratic with probability $1-\delta$ and autocratic with probability $\delta\in(0,1)$.
A democratic executive maximizes public value subject to inherited constraints.
An autocratic executive trades off erosive action against detection, contestation, and reversal under the inherited enforcement institutions.


We begin from the premise that the constitution, law and procedure do not fully specify, ex ante, how every executive action should be classified.
Let $\mathcal{A}$ denote the space of executive actions, partitioned into clearly permissible actions $\mathcal{P}$, clearly impermissible actions $\mathcal{I}$, and an ambiguity set $\mathcal{U}$ of contested cases:
\[
\mathcal{A}=\mathcal{P}\sqcup\mathcal{U}\sqcup\mathcal{I}.
\]
An operational compliance procedure maps actions to a binary administrative judgment,
\[
R_s(a)\in\{\text{Pass},\text{Flag}\},\qquad a\in\mathcal{A}.
\]
Pass means the action clears administrative review and can proceed.
Flag means it triggers delay, revision, or blocking.
This is an administrative screening judgment, not a final legal adjudication.
Taken together, these screening judgments define the alignment surface, the practical boundary between actions that pass review and actions that trigger challenge.
The model therefore does not assume that codification makes law complete.
It assumes only that the review process is stable enough to be used repeatedly, and that this stability has strategic consequences under turnover.


At $t=1$, institutions choose an adoption architecture summarized by $(x,s)$.
In administrative terms, the design variables have a straightforward interpretation. 
Scale $x$ is the administrative footprint of AI-mediated review. Higher $x$ means broader domain coverage and deeper operational reliance on that process.
Codification $s$ is the degree to which that review process becomes standardized, portable, and reproducible across cases.
Safeguards $r$ are the institutional complements that limit throughput, separate review from deployment, and speed contestation and remedy.

Codification $s\in[0,1]$ captures the intensity with which probabilistic review outputs are translated into standardized, repeatable, and recorded administrative treatment.
For an action $a\in\mathcal{A}$, the underlying system may produce a probabilistic compliance assessment rather than a categorical judgment.
The state does not act on that assessment directly.
Instead it operationalizes it through thresholds, review rules, and documentation requirements that generate the binary screening judgment $R_s(a)\in\{\text{Pass},\text{Flag}\}$.
Higher codification means narrower uncertainty bands, more uniform reason-giving, and more stable Pass/Flag determinations across time, reviewers, and domains.
As a result, feedback from review becomes more predictive across attempts and settings, which makes iterative adjustment more effective for actors with access.\footnote{For parsimony, the main text uses a single codification index $s$.  Appendix \ref{app:codification_split} separates oversight-facing auditability ($s_{\mathrm{aud}}$) from insider-facing standardization ($s_{\mathrm{std}}$); the main text's double-edged pattern concerns bundled reforms that increase both. Online Appendix \ref{app:alt_microfoundation} records an ambiguity-reduction variant.}

Safeguards, meanwhile, are slow-moving institutional complements intended to prevent abuse. 
We allow three components of safeguards, $r\equiv(r_m,r_\kappa,r_q)\in\mathbb{R}_+^3$.
The component $r_m$ captures separation of duties and access controls that limit how many variants political principals can move through the process.
The component $r_\kappa$ captures separation between evaluation and operational control, reducing how easily review feedback can be converted into deployable strategy.
The component $r_q$ captures contestation and remedy capacity, reducing the persistence of harmful actions that pass review.
Because these safeguards are embedded in formal authority, staffing, and organizational routine, we treat $r$ as inherited at both adoption and turnover.


Scaling such a workflow also creates a review-capacity constraint.
Normalize total case volume to one.
Let $m$ denote the width of the uncertainty band in the score-to-action mapping, let $q(m)$ be the share of routed cases sent to additional review with $q'(m)>0$, and let $\bar H$ denote manual-review capacity.
If a fraction $x$ of official activity is routed through the workflow, feasibility requires
\[
x q(m)\le \bar H.
\]
As scale rises, institutions cannot leave too much of the score space to case-by-case review without overwhelming capacity.
Expanding deployment therefore requires narrower uncertainty bands, more standardized treatment of recurring cases, or equivalent investments in codification.
We summarize this requirement with the feasibility constraint
\[
s\ge S(x),
\]
where $S:[0,\bar x]\to[0,1]$ is continuous and weakly increasing with $S(0)=0$ and $S(\bar x)=1$.
Because codification requires up-front coordination and ongoing maintenance, institutions that expand scale will often choose only the minimum level needed to keep the process functioning, so $s$ is frequently close to $S(x)$.\footnote{Online Appendix Lemma \ref{lem:binding_codification} provides a sufficient condition under which the adoption-stage optimum binds at $s=S(x)$.}


The outcome of interest is democratic failure when an autocrat comes to power.
By failure we mean the erosion of binding constraints that would otherwise make executive action contestable and reversible.
In such a state, elections may remain formally in place, but oversight institutions no longer reliably invalidate, deter, or remedy executive action.

Failure under autocratic turnover has two channels.
The first is overt abuse: even when actions are clearly impermissible, enforcement may be imperfect.
Let $F_0(x,s,r)\in[0,1]$ denote the probability that overtly impermissible actions succeed in leading to democratic failure when an autocrat holds office.
We assume
\[
F_{0,s}(x,s,r)\le 0,\qquad F_{0,r_j}(x,s,r)\le 0\ \text{for each }j\in\{m,\kappa,q\},
\]
reflecting that codification and safeguards can improve detectability, contestation, and remedy. 
This monotonicity assumption concerns the oversight-facing side of codification (auditability and traceability), not the insider-facing learnability component separated in 
Appendix \ref{app:codification_split}. 
We do not sign-restrict $F_{0,x}$: scale can expand opportunities for overt abuse, but it can also increase traceability and oversight capacity in some settings.
As shown in Online Appendix \ref{app:proofs}, total fragility can still rise with scale even when $F_{0,x}<0$ locally, including along the binding regime $s=S(x)$.

The second channel is within-form erosion, where the autocrat attempts to subvert the system surreptitiously. 
The ambiguity set $\mathcal{U}$ allows an autocratic executive to pursue plausibly defensible actions that nonetheless shift power in ways that weaken constraints.
When the review process is stable and feedback is available, the executive can engage in legal iteration: draft actions, observe Pass/Flag signals from internal review, courts, or oversight, revise form or sequencing, and deploy those that pass.
In this channel, the key question is not only whether an action can pass once, but whether it can remain in place long enough to produce meaningful political effects. 
Contestation and remedy capacity therefore operate mainly by shortening the expected duration of harmful passing actions, rather than by changing the initial Pass/Flag decision.\footnote{Online Appendix \ref{app:wf_microfoundation} makes this precise by modeling within-form effectiveness as passage plus persistence, with $r_q$ increasing the rate at which contested actions are reversed.}

We use the term `aligned' in the sense of oversight alignment, meaning correspondence between operational judgments and the constraints that overseers would enforce.
The risk analyzed below is that even when the operational gate is aligned, a successor can strategically navigate within the passing subset of the ambiguity set $\mathcal{U}$ so that procedurally compliant actions cumulatively weaken democratic constraints.%


To model the within-form channel, we begin with a minimal search microfoundation and then introduce a tractable rare-events approximation.
Let $N(x)\in\mathbb{R}_+$ denote the number of routinized decision interfaces exposed by scale $x$, with $N'(x)\ge 0$ and $N(0)=0$.
At each interface, an autocrat can attempt up to $M(r_m)\in\mathbb{R}_+$ variants within the consolidation window, where $M'(r_m)\le 0$.
For a single attempt, let $\rho(s,r_\kappa)\in(0,1)$ denote the probability that a proposed variant both passes the operational gate and has high erosion potential, with $\partial \rho/\partial s\ge 0$ and $\partial \rho/\partial r_\kappa\le 0$.
Let $\psi(r_q)\in(0,1)$ denote the probability that a passing erosive move persists long enough to matter, with $\psi'(r_q)\le 0$.

The per-interface probability of at least one effective erosive move is
\begin{equation}
\pi(s;r)\equiv 1-\big(1-\rho(s,r_\kappa)\psi(r_q)\big)^{M(r_m)}.
\label{eq:pi_def}
\end{equation}
Assuming conditional independence across interfaces, the probability of success somewhere in the stack is
\begin{equation}
p_{\mathrm{wf}}^{\mathrm{search}}(x,s;r)=1-\big(1-\pi(s;r)\big)^{N(x)}.
\label{eq:wf_prob_search}
\end{equation}
Define $\nu(s;r)\equiv -\ln\!\big(1-\pi(s;r)\big)\ge 0$.
Then
\[
p_{\mathrm{wf}}^{\mathrm{search}}(x,s;r)=1-\exp\!\big(-N(x)\nu(s;r)\big).
\]

To allow for vulnerability that does not operate through the standardized stack, we add an independent baseline intensity $\mu_0(r)\ge 0$.
The total search-based intensity is then:
\[
\mu^{\mathrm{search}}(x,s,r)\equiv \mu_0(r)+N(x)\nu(s;r),
\]
so that
\[
p_{\mathrm{wf}}^{\mathrm{search}}(x,s;r)=1-\exp\!\big(-\mu^{\mathrm{search}}(x,s,r)\big).
\]


For the main results, we use a benchmark specialization of this search-based intensity. 
We summarize within-form vulnerability by
\begin{equation}
\mu(x,s,r)\equiv \mu_0(r)+\eta(r)\,x\,s,
\label{eq:mu_def}
\end{equation}
where $\mu_0(r)\ge 0$ captures baseline vulnerability that exists even without a fully standardized stack, and $\eta(r)xs$ captures the additional vulnerability created when scale operates through a stable Pass/Flag gate. 
Under this benchmark, the within-form success probability is
\begin{equation}
p_{\mathrm{wf}}(x,s;r)=1-\exp\!\big(-\mu(x,s,r)\big).
\label{eq:wf_prob_main}
\end{equation}
An effective move may be either a single action or a short sequence of procedurally passing actions that, if they persist, produce durable institutional erosion.
Online Appendix \ref{app:wf_microfoundation} gives the microfoundations and approximation steps connecting the interface-level search environment to this Poisson benchmark, including a threshold variant in which erosion requires at least $k\ge 1$ effective moves within the window.
The alignment-surface logic continues to hold in that setting, with the Poisson cutoff $-\ln(1-\bar p)$ replaced by the corresponding $k$-move cutoff.

Total regime-relevant democratic failure under an autocrat occurs if either overt abuse succeeds or within-form erosion succeeds.
The resulting failure probability is
\begin{equation}
F(x,s,r)=p_{\mathrm{wf}}(x,s;r)+\big(1-p_{\mathrm{wf}}(x,s;r)\big)F_0(x,s,r)
=1-\big(1-F_0(x,s,r)\big)\exp\!\big(-\mu(x,s,r)\big).
\label{eq:failure_decomp_main}
\end{equation}
This reduced-form aggregation isolates the architectural tradeoff between administrative capacity and fragility under turnover.
Online Appendix \ref{app:alt_aggregation} reports a choice-based alternative with strategic substitution across channels.

Equation \eqref{eq:failure_decomp_main} shows why procedural modernization can be double-edged under turnover.
Codification can reduce overt abuse by improving auditability and traceability, captured by $F_{0,s}\le 0$.
But when it also makes the Pass/Flag gate more reproducible across settings, it can increase within-form vulnerability by making iterative adaptation more effective, as reflected in $\mu(x,s,r)$.
The net effect of codification on failure risk is therefore not signed in the main model.
It depends on how a reform package shifts oversight-facing auditability relative to insiders' ability to adapt to review feedback.\footnote{Online Appendix \ref{app:codification_split} makes this explicit by separating auditability ($s_{\mathrm{aud}}$) from insider-facing standardization ($s_{\mathrm{std}}$): raising $s_{\mathrm{aud}}$ is unambiguously protective, while raising $s_{\mathrm{std}}$ increases within-form risk when $x>0$.} 
If within-form erosion weakens oversight institutions or courts (for example through personnel capture or jurisdictional changes), then the channels become complementary and the separable benchmark is conservative for the paper's cautionary claims.


We now state the adoption-stage objective, which trades off administrative capacity against turnover risk through $F(x,s,r)$.
Let $G(x,s)$ denote the public value generated by the standardized administrative process in ordinary governance, with $G_x>0$ and $G_s\ge 0$.
Let $\lambda>0$ denote modernization pressure, interpreted as the marginal value of additional administrative capacity.
Let $C(x,s)$ denote implementation and operating costs, with $C_x>0$ and $C_s>0$.
Democratic failure carries a loss $\Omega>0$.
For transparency, we take expected loss to be proportional to the failure probability $F$; the results extend immediately to any smooth increasing loss function of $F$.

At $t=0$, safeguards can be chosen or reformed at convex cost $B(r)$, with $B(0)=0$, $B_{r_j}>0$, and convexity reflecting the political and organizational frictions involved in strengthening safeguards.
The results below focus on the adoption-stage choice conditional on inherited $r$.
We treat the adoption-stage objective as a reduced-form representation of good-faith institutional choice.
The actor labeled ``institutions'' can be read as a legislature, an agency, or a coalition that internalizes at least part of the downstream cost of democratic breakdown.
This representation is most plausible when current decision-makers expect future political competition and face legal, electoral, or reputational costs from institutional failure.
Under polarization, short horizons, or incumbents expecting to benefit from within-form advantages, threshold crossing and incomplete post-crisis unwinding become more likely, and safeguards are less likely to keep pace.

The capacity term $\lambda G(x,s)-C(x,s)$ is not multiplied by $(1-\delta)$ because it captures benefits that are realized and valued in routine governance regardless of who later holds office (service delivery, legitimacy, and administrative performance),
while the turnover-sensitive term is the expected democratic survival loss $\delta\Omega F(x,s,r)$.

The adoption-stage objective is
\begin{equation}
U_1(x,s;r,\lambda)=\lambda\,G(x,s)-C(x,s)-\delta\,\Omega\,F(x,s,r),
\label{eq:adoption_objective_main}
\end{equation}
subject to $x\in[0,\bar x]$, $s\in[0,1]$, and $s\ge S(x)$.

\section{Results}
\label{sec:results}

Democratic delegation is easiest to control when official action runs through rules, tests, and records.
Those features make departures from law or procedure easier to detect, challenge, and reverse.
Under turnover, however, the same administrative process can also become subject to analysis and exploitation by the new government. 
This section develops three linked results.
The first shows how scale, codification, and safeguards shape the probability that an autocratic successor can find a politically useful package that still passes review.
The second shows how additional codification can shift from constraining abuse to making insider iteration easier. 
The third shows that shifts induced by temporary modernization pressure need not reverse once that pressure subsides.
Such temporary pressure can push institutions across the within-form threshold, and later democratic governments may repair only part of the inherited workflow.

The central dynamics lie in the within-form channel of abuse. 
Here, a successor government inherits an approval process and iterates against it in the ambiguous region $\mathcal{U}$. 
Scale expands the number of routinized decisions and review points where such probing can occur. 
Codification makes pass-or-flag feedback more stable and more reusable across attempts and domains. 
Safeguards reduce that iterative capacity by limiting unilateral throughput ($r_m$), weakening test-to-deploy transferability ($r_\kappa$), and shortening the window for exploitation through faster contestation and remedy ($r_q$). 
The first result formalizes this logic and characterizes the surface available for the leader to attack, a surface created by the same effort to keep AI-mediated administration legible to review.

\begin{proposition}[Scale, safeguards, and the alignment surface]
\label{prop:alignment_surface_scale}

Let within-form success be
\[
p_{\mathrm{wf}}(x,s;r)=1-\exp\!\big(-\mu^{\mathrm{search}}(x,s,r)\big),
\qquad
\mu^{\mathrm{search}}(x,s,r)\equiv \mu_0(r)+N(x)\nu(s;r),
\]
where $N'(x)\ge 0$, $\nu_s(s;r)\ge 0$, and safeguards weakly reduce within-form intensity in the sense that
$\mu_{0,r_j}(r)\le 0$ and $\nu_{r_j}(s;r)\le 0$ for each safeguard component $r_j$.

Then $p_{\mathrm{wf}}(x,s;r)$ is weakly increasing in scale $x$ and codification $s$, and weakly decreasing in each safeguard component $r_j$.
\end{proposition}

For any chosen concern level $\bar p\in(0,1)$, the condition $p_{\mathrm{wf}}(x,s;r)\le \bar p$ is equivalent to a single-index bound on within-form intensity:
\begin{equation}
\mu_0(r)+N(x)\nu(s;r) \le -\ln(1-\bar p).
\label{eq:alignment_surface_condition}
\end{equation}
Again, the alignment surface is the inherited pass-or-flag boundary created by the review process.
This equation gives the threshold condition under which the inherited boundary becomes politically dangerous by making within-form success sufficiently likely to matter.
Below the threshold, a successor government is unlikely to find a passing erosive package before contestation catches up.
Above it, routinized decisions and stable review signals make successful probing more likely.
Safeguards push that threshold outward by reducing baseline exposure $\mu_0(r)$ and by making repeated insider iteration less effective through $\nu(s;r)$.

For the remainder of this section, we specialize the general search formulation to the Poisson benchmark introduced in Section \ref{sec:model}.
In that benchmark, within-form success takes the exponential-search form
\[
p_{\mathrm{wf}}(x,s;r)=1-\exp\!\big(-\mu(x,s,r)\big),
\]
and within-form intensity is linear in the interaction between scale and codification:
\[
\mu(x,s,r)=\mu_0(r)+\eta(r)xs,
\]
with $\eta(r)>0$.
Here $\mu_0(r)$ captures within-form vulnerability that exists even without a fully standardized process, while $\eta(r)xs$ captures the additional vulnerability created when decisions are made at scale through stable criteria that officials can learn to navigate.

Under this benchmark, the alignment-surface condition has an immediate threshold interpretation.
For any target $\bar p\in(0,1)$ and any $(r,s)$ with $\eta(r)s>0$, there is a critical scale
\begin{equation}
x^{\mathrm{crit}}(\bar p; r,s)\equiv \max\!\left\{0,\frac{-\ln(1-\bar p)-\mu_0(r)}{\eta(r)s}\right\},
\label{eq:xcrit_scale}
\end{equation}
above which within-form success exceeds the threshold: $x\ge x^{\mathrm{crit}}(\bar p;r,s)\Rightarrow p_{\mathrm{wf}}(x,s;r)\ge \bar p$.
Substantively, $x^{\mathrm{crit}}$ marks the point at which routinized administration gives a successor enough opportunities to test the process, and stable enough feedback to learn from those attempts, that finding at least one procedurally passing erosive package becomes likely.

The implication is straightforward.
As scale rises, a successor has more opportunities to test inherited procedures, and codification makes the feedback from those tests easier to reuse. 
In the benchmark, that logic is captured by the interaction term $\eta(r)xs$. 
(Online Appendix~\ref{app:proofs} derives the comparative static of total failure with respect to scale while holding codification fixed in \eqref{eq:dFdx}, and gives a sufficient condition in \eqref{eq:scale_condition}. )
Because scaling in practice usually requires additional standardization to keep administration legible and auditable, we also impose the feasibility requirement $s\ge S(x)$ and report the induced change in failure along the binding path $s=S(x)$, including the total derivative \eqref{eq:dFdx_binding} and sufficient condition \eqref{eq:scale_condition_binding}.

Figure \ref{fig:core_threshold_crossing_main} summarizes the geometry of the first and third results.
Rising modernization pressure increases adopted scale, feasibility then pushes codification upward along the binding path $s=S(x)$, and the polity becomes exposed once that path crosses the within-form vulnerability threshold.

\begin{figure}[t] 
\centering
\includegraphics[width=0.74\textwidth]{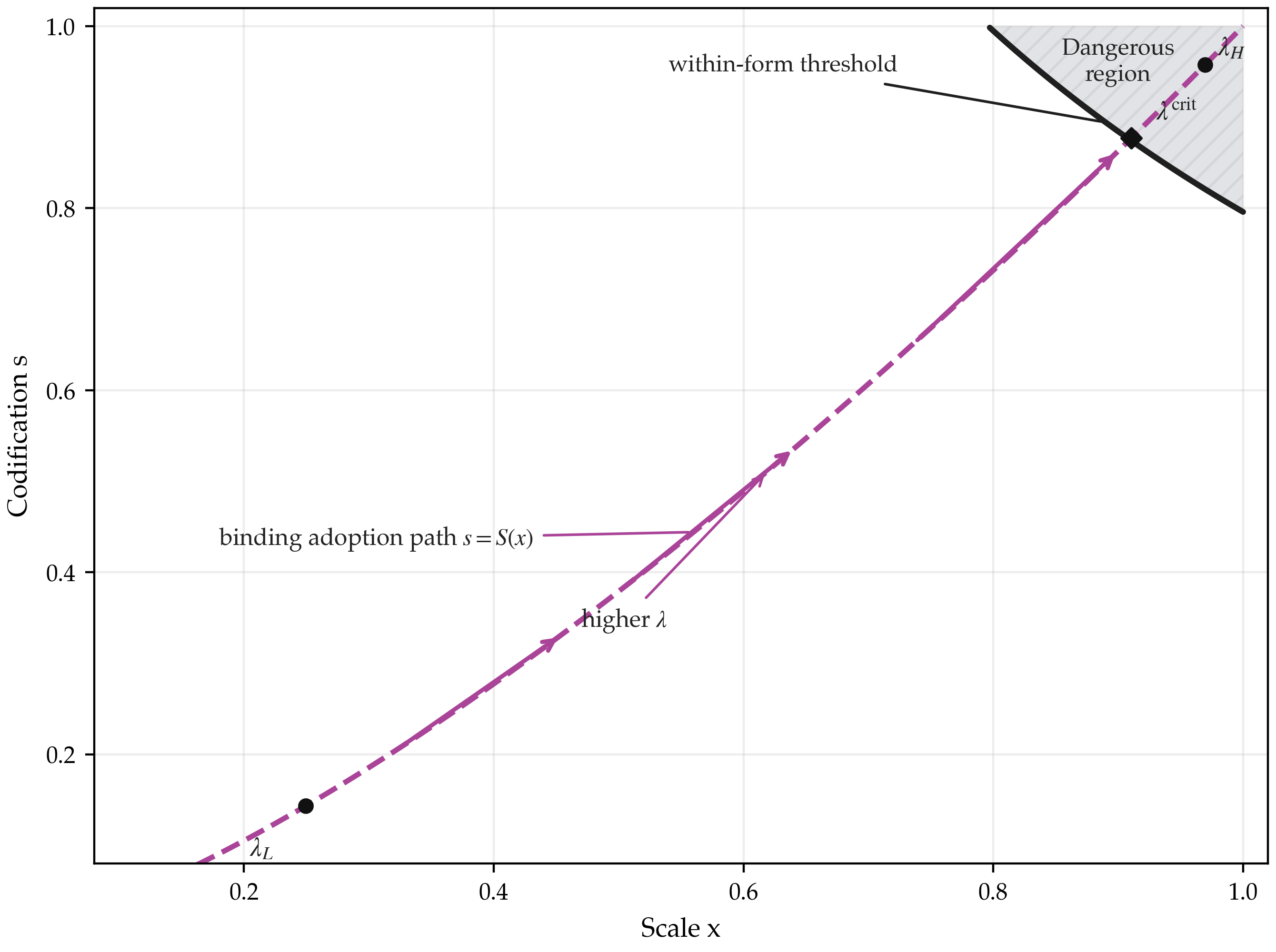}
\caption{Modernization pressure and threshold crossing along the binding adoption path.}
\label{fig:core_threshold_crossing_main}
\floatnote{The solid curve marks the threshold $p_{\mathrm{wf}}(x,s;r)=\bar p$ in $(x,s)$-space for a fixed safeguard bundle $r$.
The dashed curve is the binding adoption path $s=S(x)$.
As modernization pressure $\lambda$ rises, the adopted architecture moves along that path from $\lambda_L$ to $\lambda_H$.
When $\lambda_L<\lambda^{\mathrm{crit}}\le \lambda_H$, the path crosses the threshold at $\lambda^{\mathrm{crit}}$ and enters the shaded region, where within-form abuse becomes politically consequential.
Schematic illustration of Proposition \ref{prop:alignment_surface_scale} and Proposition \ref{prop:scale_increases_lambda}.}
\end{figure}

Next we compare how codification changes the relative importance of the overt and within-form channels of abuse.
Conditional on an autocratic successor, democratic failure can occur either through overtly impermissible abuse, with probability $F_0(x,s,r)$, or through within-form erosion, with probability $p_{\mathrm{wf}}(x,s;r)$.
Under the separable aggregation in Section \ref{sec:model}, total failure is
\[
F(x,s,r)=1-\big(1-F_0(x,s,r)\big)\exp\!\big(-\mu(x,s,r)\big),
\]
so codification affects $F$ through two margins that pull in opposite directions. 
Higher $s$ tends to reduce overt vulnerability $F_0(x,s,r)$ by making departures easier to detect and contest. 
But higher $s$ can also increase within-form vulnerability by making the compliance boundary more stable and lessons from repeated attempts more transferable.
The first margin has diminishing returns.
Early codification removes blatant violations and makes meaningful review possible.
Once procedures are already routinized, further codification yields smaller deterrent gains.

\begin{proposition}[Codification and the shift from constraint to exploitability]
\label{prop:codification_flip_main}

Fix $(x,r)$ and work under the Poisson benchmark.
Suppose $F_{0,s}(x,s,r)\le 0$ and $F_{0,ss}(x,s,r)\ge 0$ on $s\in[0,1]$.
Then the marginal effect of codification on total failure changes sign at most once on $[0,1]$.
If $h(0)<0<h(1)$, there exists at least one cutoff $s^{\mathrm{flip}}(x,r)\in(0,1)$ at which marginal codification switches from protective to risk-increasing.
If, in addition, $x>0$ and $F_{0,s}(x,s,r)<0$ for interior $s$, the cutoff is unique.
\end{proposition}

\noindent Proposition \ref{prop:codification_flip_main} concerns bundled reforms that often increase both oversight-facing auditability and insider-facing boundary standardization. The reversal arises when additional codification moves both margins together: overt vulnerability continues to fall, but at a diminishing rate, while within-form exploitability rises as the approval boundary becomes more stable and transferable.

\noindent Under the Poisson benchmark $\mu(x,s,r)=\mu_0(r)+\eta(r)xs$, this tradeoff is transparent.
Differentiating total failure yields
\begin{equation}
\frac{\partial F}{\partial s}
=
\exp\!\big(-\mu(x,s,r)\big)
\Big(F_{0,s}(x,s,r)+\big(1-F_0(x,s,r)\big)\eta(r)x\Big),
\label{eq:codification_tradeoff}
\end{equation}
so the sign of the marginal effect is governed by
\[
h(s)\equiv F_{0,s}(x,s,r)+\big(1-F_0(x,s,r)\big)\eta(r)x.
\]
Because $\exp(-\mu)>0$, codification becomes locally risk-increasing if and only if
\begin{equation}
\big(1-F_0(x,s,r)\big)\eta(r)x > -F_{0,s}(x,s,r).
\label{eq:codification_cutoff}
\end{equation}
Equation \eqref{eq:codification_cutoff} states the reversal directly: additional codification becomes dangerous when its contribution to within-form exploitability outweighs its remaining contribution to deterring overt abuse.
For details, see Online Appendix \ref{app:robustness_codification}.

While agencies may have good administrative reasons to codify procedures in order to use AI for efficiency gains, Proposition \ref{prop:codification_flip_main} shows that the aggregate political consequences of those incremental steps need not move in a single direction.

The curvature assumption $F_{0,ss}\ge 0$ captures a familiar institutional pattern.
Early formalization removes blatant violations and makes review feasible.
Later rounds often add process and paperwork faster than they add new enforcement capacity.
The reversal is most salient in high-scale settings, where baseline auditability may already be in place and further codification mainly improves the learnability and reuse of procedurally passing tactics.
Figure \ref{fig:one_crossing_flip_main} provides a schematic illustration of the second result.

\begin{figure}[t] 
\centering
\includegraphics[width=0.78\textwidth]{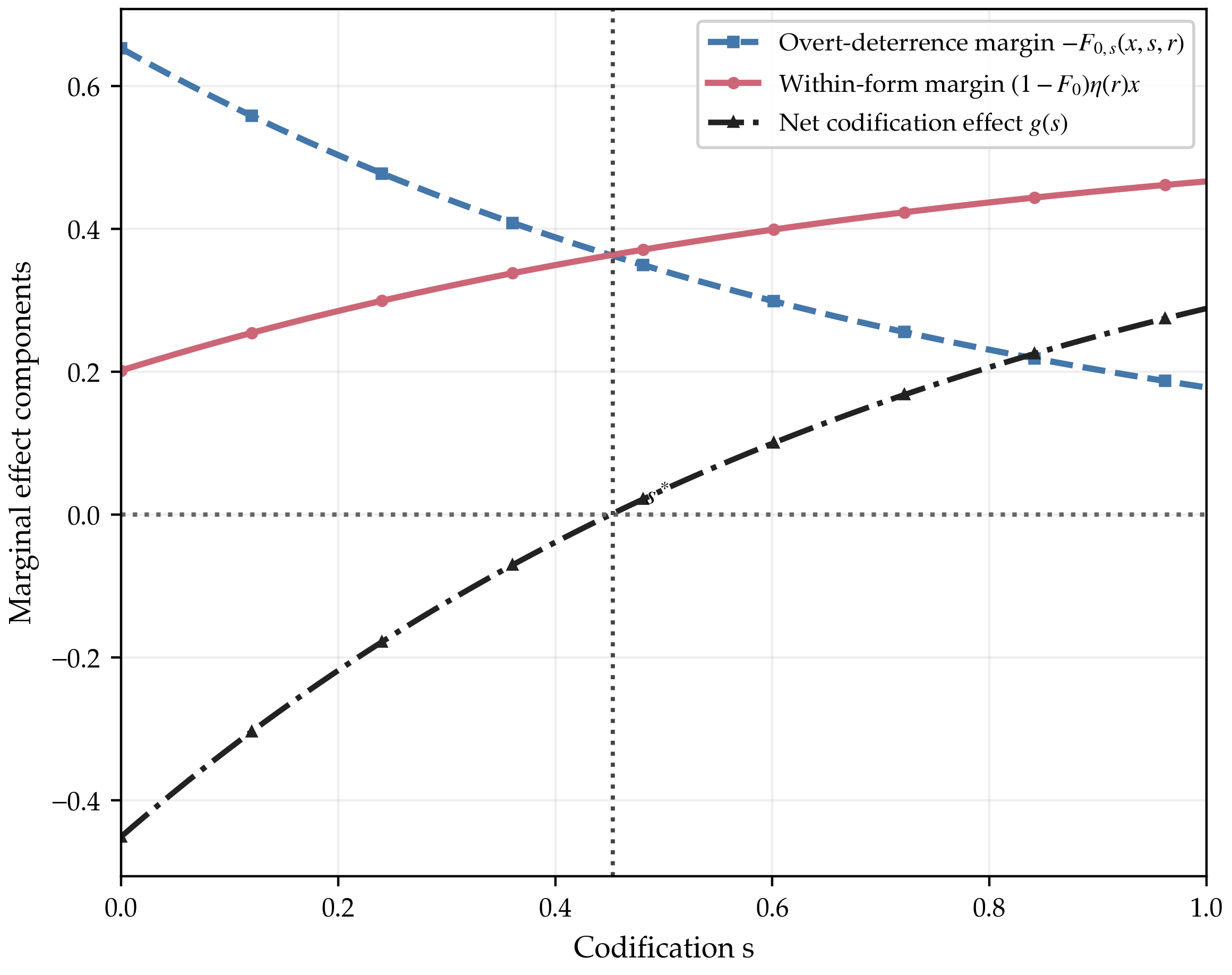}
\caption{Codification and the Shift from Constraint to Exploitability}
\label{fig:one_crossing_flip_main}
\floatnote{Additional codification is initially useful for deterring overt abuse, but eventually generates a more legible surface for within-form abuse.
Horizontal axis: codification intensity $s$.
Vertical axis: marginal-effect components.
Dashed line: overt-deterrence margin $-F_{0,s}(x,s,r)$.
Solid line (red): within-form margin $(1-F_0(x,s,r))\eta(r)x$.
Dash-dot line (black): net codification effect $h(s)$.
The crossing at $s^{\mathrm{flip}}$ marks where marginal codification switches from protective to risk-increasing.
Schematic illustration of Proposition \ref{prop:codification_flip_main}.}
\end{figure}

The underlying tradeoff becomes clearer if the reduced-form codification variable is decomposed into two analytically distinct components: auditability $s_{\mathrm{aud}}$, which improves legibility and contestation for overseers, and boundary standardization $s_{\mathrm{std}}$, which makes approval criteria more stable and transferable for insiders. 
Under that separation, the two components have clean, one-signed effects:
\[
\begin{aligned}
\frac{\partial F}{\partial s_{\mathrm{aud}}}
&=
\big(1-p_{\mathrm{wf}}(x,s_{\mathrm{std}};r)\big)\,F_{0,s_{\mathrm{aud}}}(x,s_{\mathrm{aud}},r)\le 0,
\\
\frac{\partial F}{\partial s_{\mathrm{std}}}
&=
\big(1-F_0(x,s_{\mathrm{aud}},r)\big)\exp\!\big(-\mu(x,s_{\mathrm{std}},r)\big)\,\eta(r)x\ge 0.
\end{aligned}
\]

On this account, auditability reduces risk by tightening the overt channel, whereas boundary standardization increases risk by making within-form exploitation easier. 
Online Appendix \ref{app:codification_split} unbundles that common reform margin into oversight-facing auditability and insider-facing standardization, showing which part of the bundle does which work.
This decomposition also clarifies the practical design problem.
In practice, preserving the gains from auditability without also increasing iterability requires complementary safeguards, such as durable access controls and interface rules that limit how easily passing strategies can be reused across domains.

A natural objection is that codification can narrow the gray zone itself. 
By clarifying doctrine or reducing the set of contested cases that remain practically exploitable, it may reduce both overt abuse and within-form erosion. 
Our results focus on the politically relevant range in which that contraction of ambiguity is limited, though the Appendix relaxes this assumption.\footnote{Online Appendix \ref{app:alt_microfoundation} allows codification to shrink exploitable ambiguity by scaling the standardized within-form term by $\omega(s)\in(0,1]$ with $\omega'(s)\le 0$, so that $\mu(x,s,r)=\mu_0(r)+\eta(r)\,x\,s\,\omega(s)$. In that formulation, the within-form effect of codification remains risk-increasing when $s\,\omega(s)$ does not fall with $s$; equivalently, when $\frac{d}{ds}[s\,\omega(s)]\ge 0$. If ambiguity shrinks fast enough, it can attenuate or overturn the within-form mechanism and thereby mute or eliminate the reversal of the marginal effect.}

The result rests on two broad features.
First, within-form success rises with the intensity of search made possible by standardization.
Second, once basic auditability is in place, further codification yields diminishing marginal gains in deterring overt abuse.\footnote{The logic does not depend on the Poisson benchmark or on the linear $xs$ index. It requires only a single-index mapping from within-form intensity $\mu(x,s,r)$ into success $p_{\mathrm{wf}}(\mu)$, with $p_{\mathrm{wf}}$ increasing in $\mu$; see Online Appendix \ref{app:nonlinear_mu}. Online Appendix \ref{app:wf_microfoundation} considers threshold formulations in which erosion requires at least $k\ge1$ effective moves; the key condition remains intensity below a cutoff, with $-\ln(1-\bar p)$ replaced by the corresponding threshold $\tau(\bar p)$. The marginal leverage of overt-only reforms can be aggregation-dependent under a choice-based alternative in which an autocrat substitutes strategically across channels; see Online Appendix \ref{app:alt_aggregation}. Formal derivative and dominance calculations are collected in Online Appendix \ref{app:proofs}.}

Two considerations sharpen the implications in real-world institutional settings.
First, the benchmark treats risk as accumulating across decision points in a simple additive way, as if openings in different parts of the administrative process were distinct.
In practice, vulnerabilities may be correlated because deployments often share tooling, training data, evaluation routines, and guardrails.
Evidence on universal and transferable attacks in the AI literature suggests that a tactic discovered in one system can sometimes be used against models built by other actors \citep{zou2023universal,huang2025stronger}. 
This makes scale potentially more dangerous than the benchmark implies, because expanding deployment may not merely create more sites of possible failure. 
It can also allow successful tactics to travel across shared routines and common infrastructure, including across agencies using different models or providers.

Second, the benchmark also abstracts from the unique and double-edged role of human discretion in bureaucracy. 
A leader may exploit a pliable or compromised official, but efforts to bend administration can also be disrupted by whistleblowers, refusing subordinates, leaks, or inter-office rivalry.
These forms of resistance are irregular, but they make subversion depend on the continued cooperation of particular people.
By contrast, LLM- or agent-mediated processes can make strategic adaptation less person-dependent and more procedure-dependent.
When passing strategies can be learned against stable administrative criteria, exploitation becomes easier to iterate and reuse, and less exposed to interruption by a single dissenting official.

For design purposes, it helps to distinguish safeguards by the part of the failure process they affect.
In our parameterization, safeguards shape total failure $F$ through three outcome-relevant terms: overt vulnerability $F_0(x,s,r)$, baseline within-form intensity $\mu_0(r)$, and standardized within-form intensity $\eta(r)$.
Reforms that mainly reduce $F_0$ matter most when failure still occurs through visible departures from the process.
Once within-form success is already likely, however, those reforms lose leverage.
By contrast, reforms that reduce $\mu_0$ or $\eta$ constrain within-form search itself, by limiting access, reducing the transferability of successful tactics across settings, or making contestation and remedy faster.
Where institutional design can lower all three terms at once, total failure falls weakly throughout.

This argument also has a temporal dimension. 
The political problem is not simply that safeguards may be weak.
It is that governments can expand AI-mediated administration faster than they can build the rules, staffing, and remedies needed to keep that expansion under democratic control. 
Even a government acting in good faith, and anticipating turnover, may scale deployment before those protections are in place, because capacity gains arrive quickly, whereas many protections require statutory authority, independent staffing and funding, separation of duties, and the capacity to detect, challenge, and remedy abuse. 
Unfortunately, such protections are often developed only in response to visible failures or scandal. 
Online Appendix \ref{app:illustration} shows why this timing asymmetry matters in the benchmark.
When within-form vulnerability takes the form $\mu(x,s,r)=\mu_0(r)+\eta(r)xs$, the protective value of safeguards that reduce $\eta(r)$ rises with the size of the standardized administrative footprint $xs$.
As a result, holding safeguards fixed becomes more dangerous as deployment and codification expand.

The next proposition shows how ordinary pressure to modernize administration can create a new political risk even when safeguards do not weaken.
When governments face stronger pressure to cut costs, reduce delay, or standardize decisions, they adopt more AI-mediated administration and, on the binding regime, more codification as well.
If safeguards do not rise in step, a government can move towards a setting in which strategic exploitation is becomes likely because administrative expansion has outpaced institutional protection.


\begin{proposition}[Modernization pressure can push adoption across the within-form threshold when safeguards are fixed]
\label{prop:scale_increases_lambda}
Fix safeguards at bundle $r$ and work along the binding regime $s=S(x)$:
\[
U_1(x,S(x);r,\lambda)=\lambda\,G(x,S(x))-C(x,S(x))-\delta\,\Omega\,F(x,S(x),r).
\]

\smallskip
\noindent\emph{(i) Pressure increases adopted scale (and codification on the binding regime).}
Suppose that for this fixed $r$ and over the relevant range the objective is strictly concave in $x$ and admits an interior maximizer $x^*(\lambda;r)$.
Assume also that $\frac{d}{dx}G(x,S(x))>0$ (for example, under the maintained assumptions $G_x>0$, $G_s\ge 0$, and $S'(x)\ge 0$).
Then $x^*(\lambda;r)$ is strictly increasing in $\lambda$.
Moreover, if $S$ is weakly increasing, the implied codification choice $s^{\mathrm{adopt}}(\lambda;r)=S(x^*(\lambda;r))$ is weakly increasing in $\lambda$.

\smallskip
\noindent\emph{(ii) Threshold implication under fixed safeguards (Poisson benchmark).}
Fix any target threshold $\bar p\in(0,1)$ and define
\[
x^{\mathrm{crit}}(\bar p;r)\equiv \inf\big\{x\in[0,\bar x] : p_{\mathrm{wf}}(x,S(x);r)\ge \bar p\big\}.
\]
If there exist $\lambda_L<\lambda_H$ such that $x^*(\lambda_L;r)<x^{\mathrm{crit}}(\bar p;r)$ and $x^*(\lambda_H;r)>x^{\mathrm{crit}}(\bar p;r)$,
then there exists $\lambda^{\mathrm{crit}}\in(\lambda_L,\lambda_H]$ such that
\[
p_{\mathrm{wf}}\big(x^*(\lambda^{\mathrm{crit}};r),S(x^*(\lambda^{\mathrm{crit}};r));r\big)= \bar p,
\]
and for all $\lambda\ge \lambda^{\mathrm{crit}}$ it satisfies $p_{\mathrm{wf}}(x^*(\lambda;r),S(x^*(\lambda;r));r)\ge \bar p$.
\end{proposition}

\noindent\emph{Proof sketch.}
For part (i), the interior first-order condition along $s=S(x)$ and strict concavity imply, via the implicit function theorem, that $\frac{dx^*(\lambda;r)}{d\lambda}>0$ when $\frac{d}{dx}G(x,S(x))>0$.
For part (ii), define $\zeta(\lambda)\equiv p_{\mathrm{wf}}(x^*(\lambda;r),S(x^*(\lambda;r));r)$.
Under the stated assumptions, $\zeta$ is continuous and weakly increasing in $\lambda$, so the bracketing condition implies existence of $\lambda^{\mathrm{crit}}$ by the intermediate value theorem, and monotonicity gives $\zeta(\lambda)\ge\bar p$ for all $\lambda\ge \lambda^{\mathrm{crit}}$.
Online Appendix \ref{app:proofs} provides the full argument.

\noindent Evaluated at the crossing point where $p_{\mathrm{wf}}=\bar p$, total failure satisfies
\[
\begin{aligned}
F\big(x^*(\lambda^{\mathrm{crit}};r),S(x^*(\lambda^{\mathrm{crit}};r)),r\big)
&=
F_0\big(x^*(\lambda^{\mathrm{crit}};r),S(x^*(\lambda^{\mathrm{crit}};r)),r\big)
\\
&\qquad
+ \Big(1-F_0\big(x^*(\lambda^{\mathrm{crit}};r),S(x^*(\lambda^{\mathrm{crit}};r)),r\big)\Big)\bar p.
\end{aligned}
\]
by \eqref{eq:failure_decomp_main}.
At that point, overt vulnerability still matters, but the polity has also entered a region in which within-form exploitation is no longer negligible.

Proposition~\ref{prop:scale_increases_lambda} shows the potential danger posed by pressure to modernize administration. 
A temporary push to cut costs, reduce delay, or standardize decisions can increase adopted scale and codification enough to move the system into a riskier region.
That movement may also be difficult to reverse once the pressure spike passes. 
In practice, later governments rarely rebuild the administrative stack from scratch.
They inherit installed scale, review procedures, and routines already embedded in service delivery.
The threshold-crossing logic does not depend on the Poisson benchmark or on the linear $xs$ index; Appendix I shows that it extends to any monotone within-form intensity index.

This makes the next problem one of selective repair.
Democratic actors confronting a risky inherited system rarely want less recordkeeping or weaker review.
Rather, they want to preserve auditability while reducing the portability and reuse that make the workflow easier to exploit from within.
 Appendix~\ref{app:codification_split} formalizes this distinction with a local split between oversight-facing auditability $s_{\mathrm{aud}}$ and insider-facing standardization $s_{\mathrm{std}}$.
At adoption, common reform packages often increase both together.
After a crisis, by contrast, repair primarily targets the part of the workflow that makes successful tactics portable and reusable. 
For the persistence extension, installed scale and auditability are therefore treated as temporarily sunk, while repair lowers insider-facing standardization.

Whenever baseline within-form vulnerability is not already above the concern threshold, there is a critical level of insider-facing standardization above which within-form abuse becomes likely. 
Suppose a temporary high-pressure episode leaves the polity above that threshold. 
After pressure subsides, a democratic government may find some refactoring worthwhile, yet still stop short of enough repair to move fully back below the threshold once marginal refactoring costs exceed marginal benefits in the low-risk region. 
The result is partial but incomplete post-crisis repair. 
The inherited workflow is rolled back to some degree, but not enough to restore the safer architecture that would have been chosen from scratch under ordinary conditions. 
Appendix~\ref{app:post_crisis_unwinding} states and proves the formal result.

Taken together, these results describe a path-dependent problem of democratic control in a technocratic bureaucracy. 
Codification and scale can harden constraints by improving detectability and contestation, but they can also make it easier to learn and reuse successful strategies for within-form abuse.
Temporary modernization pressure can move institutions into that riskier region before safeguards catch up, and once a large standardized administrative stack is installed, later democratic governments may only partially unwind it.
Democratic resilience therefore depends not only on where adoption first lands, but also on whether institutions can preserve auditability while keeping insider-facing standardization, and the costs of restructuring inherited workflows, below the point at which temporary expansions become durable sources of vulnerability.
The next section uses this logic to discuss design levers that preserve auditability while limiting opportunities for within-form abuse, and to identify empirical signatures of incomplete post-crisis repair.

\section{Discussion}
\label{sec:discussion}

Procedural modernization is pursued for a familiar democratic reason: when delegated decisions are structured by rules, records, review, and contestation, they become easier to supervise and challenge \citep{mccubbins1987administrative,mccubbins1989structure,scott1998seeing,power1997audit,kroll2017accountable}. 
As probabilistic AI systems are incorporated into public administration, demand for verifiability rises: officials want outputs they can defend against statutes, constitutions, and procedure, and they want audit trails that make departures detectable and attributable. 
Our results highlight a paradox under adverse political turnover. 
Design choices made to make administration more reviewable can also reorganize authority inside the state in ways that later advantage a successor willing to leverage the inherited process to their strategic advantage. 
The same compliance apparatus that deters blatantly impermissible acts can also expose a learnable approval boundary that insiders can probe, map, and optimize against in contested domains where legality is ambiguous and remedies take time. 
The distinctive vulnerability is therefore not hidden illegality, but within-form erosion that clears procedure inside contested legality.

This learnability mechanism differs from standard accounts of bureaucratic control. 
In conventional principal-agent stories, rules and recordkeeping reduce slack by increasing observability and sanctionability \citep{huber2002deliberate}. 
Here, standardization can shift advantage toward whoever controls iteration through the compliance process. 
Under incomplete contracts, even when review is an imperfect guide to underlying legality, it can still provide a predictable signal of which actions, justifications, or submissions are likely to receive approval. 
When review relies on such a signal, actors learn to optimize for the criteria that govern approval rather than for the underlying objective those criteria imperfectly represent. 

A successor with operational authority can repeatedly draft and revise policy memos, eligibility criteria, enforcement priorities, procurement standards, or personnel procedures, observe feedback generated by internal review and compliance responses, and then refine the form, sequencing, and justification until the action is judged to pass. 
Critically, the informational asymmetry may be selective.
Even when oversight institutions retain formal authority, external actors may find it difficult to understand the operational boundary, while insiders learn it quickly through repeated interaction \citep{gorsuch2024overruled}. 
Thus, the compliance architecture changes the relative speed and fidelity with which different actors can learn where permissibility lies.

The alignment-surface result translates this idea into an architecture-level threshold.
When the compliance-governed stack is limited in scope, or when the approval boundary is unstable enough that feedback does not transfer cleanly across attempts and domains, within-form abuse potentials are difficult to discover and scale. 
As the stack grows and the boundary becomes stable and transferable, the probability that at least one regime-relevant within-form package is found rises sharply. 
This threshold logic helps explain why modernization can appear locally beneficial while creating growing systemic fragility. 
Incremental expansions that each seem justified on administrative grounds can cumulatively move institutions across the frontier without any single decision appearing to change the constitutional order.

This dynamic explains why procedural automation can eventually become counterproductive. 
Early investments in standardization strengthen oversight against overtly impermissible acts by improving documentation, traceability, and the feasibility of review. 
As those mechanisms mature, however, marginal oversight gains from further procedural refinement can diminish, while marginal gains in insider iterability can continue to rise because the approval boundary becomes more reproducible and portable. 
Additional codification can then shift abuse away from overt violations and toward within-form strategies that preserve the appearance of legality \citep{bermeo2016backsliding,varol2015stealth,scheppele2018autocratic}.

Modernization explains how these risks can accumulate in practice. 
Agencies and legislatures face recurring pressure to deliver services, respond to crises, and administer complex programs at scale. 
Adoption is therefore often a sequence of locally defensible expansions rather than a single constitutional choice. 
Proposition~\ref{prop:scale_increases_lambda} shows how such pressure can move institutions across the within-form threshold before the surrounding safeguards catch up.

The persistence extension sharpens the point. 
Later democratic governments do not usually rebuild the administrative stack from scratch. 
They inherit installed scale, shared review artifacts, and routines already embedded in service delivery. 
For that reason, rollback is often selective rather than complete. 
Governments may preserve audit trails, documentation, and reviewability while trying to reduce the portability and reuse of internal compliance feedback. 
Where refactoring costs rise with installed scale, that repair can be real yet incomplete. 
Temporary pressure can therefore leave durable inherited vulnerability even after ordinary political conditions return.

We suggest specific reasons why safeguards cannot simply be assumed to keep pace with capacity. 
First, throughput constraints limit unilateral iteration by reducing the number of variants a principal can submit for review within a consolidation window.
Second, decoupling evaluation from deployment reduces transferability: when the process that generates compliance feedback mirrors the pathway that implements policy, feedback is easier to convert into deployable strategy; institutionally independent review and controls on experimentation make that conversion more difficult. 
Third, remedy and contestation capacity can limit persistence to the extent that they can reverse attempted abuse quickly and reliably. 
These complements are hard to scale because they require durable coalitions, institutional capacity, and often independent authority. 
They are also politically asymmetric: defenders must secure many potential points of vulnerability, while a successor needs only one workable within-form exploit that they can put to work. 
Finally, in deep governance domains, the relevant constraints are not fully specifiable in advance. 
Codification can clarify doctrine and reduce degrees of freedom, but it cannot eliminate contested judgments rooted in value conflict and evolving interpretation. 
Efforts to operationalize such judgments may therefore make review more legible and predictable in ways that facilitate strategic adaptation.


The mechanism should be strongest where AI is embedded in consequential approval or deployment workflows, where internal feedback is informative enough to guide adaptation, where standardized review artifacts are reused across many cases or domains, and where evaluation is tightly coupled to deployment. 
It should be weaker where AI remains advisory, access to iterative feedback is tightly controlled, or review is institutionally decoupled from implementation.

These scope conditions also discipline the evaluation of reform proposals. 
The model suggests a criterion for evaluating reform that does not collapse into generic calls for ``more oversight.'' 
Measures that make decisions more visible and reviewable can still increase risk if they also make the approval process more legible to insiders or allow contested approvals to persist long enough to matter. 
Reform should therefore be assessed in terms of how it changes opportunities for strategic adaptation under future political control, rather than only in terms of ex post supervision.

This also clarifies how the argument differs from concerns about older administrative scoring tools. 
Deterministic rule systems and narrower expert models could also routinize decisions, but they typically operated in more limited domains or embedded the operative rule more directly in the decision itself. 
Probabilistic generative systems are distinctive because they can be deployed across heterogeneous and more open-ended tasks while still requiring an external score-to-action layer that converts uncertain outputs into administrative action aligned to Constitutional, legal, and operational rules. 
While some fear that AI-governed administration could become Kafka-esque, and others expect logging and formalization to make it fully transparent, the concern here is different from both: systems can become more auditable while also more learnable to insiders. 
Because probabilistic models raise concerns about arbitrariness and defensibility, organizations often respond by demanding reproducibility, standardized rubrics, and traceable justifications. 
That combination of broad applicability and operationalization pressure intensifies the demand for standardization, reviewability, and reusable compliance artifacts.

The central observable quantities are therefore scope, boundary stability, and the institutional environment that governs iteration, decoupling, and persistence. 
Following executive turnover, jurisdictions or agencies with broader penetration of automated processes and more reusable compliance rubrics should exhibit a greater share of consequential actions that are procedurally successful despite contestation, relative to otherwise similar settings. 
The persistence extension adds a further implication: temporary crises or performance shocks should be followed by partial rollback rather than full restoration, especially where installed bases are large. 
One should therefore look for durable increases in routinized scope, reusable review artifacts, or tightly coupled review-and-deployment workflows even after the original pressure episode subsides. 
Where internal records are accessible, one should also observe traces of iteration and selective repair, including strategic sequencing, incremental reframing around compliance boundaries, segmentation of some workflows without full decoupling, and unequal rollback across units that inherited different installed bases.

Several modeling choices likely render the argument conservative. 
Adoption is modeled as a good-faith tradeoff that internalizes turnover risk; if adopters are myopic, polarized, or expect to benefit from within-form advantages, threshold crossing and incomplete unwinding become more likely. 
The main text treats overt and within-form channels as separable conditional on architecture; if within-form erosion weakens oversight institutions and thereby increases overt vulnerability, complementarities strengthen the caution. 
These extensions underscore the broader implication: procedural modernization can reduce overt illegality while increasing the feasibility of procedurally clean consolidation at scale, and in AI-governed compliance systems this possibility should be treated as a first-order institutional risk under turnover.

\section{Conclusion}
\label{sec:conclusion}

Governments are increasingly interested in using AI to make administration cheaper, faster, and more consistent.
But probabilistic systems can be used in public administration only if their outputs can be checked against law, procedure, and official records.
This paper has argued that the same work required to make AI governable can also create a political vulnerability under turnover.
In contested legal domains, efforts to make official decisions more reviewable can produce stable approval criteria and repeatable procedures that insiders learn to navigate.
After turnover, a successor may be able to revise actions until they clear those inherited procedures, preserve the appearance of lawful rule, and still shift authority in consequential ways.

The model highlights three implications. 
First, the danger is nonlinear.
As AI-mediated procedures spread across more consequential domains, and as review criteria become more stable across repeated attempts, the chance that a successor with illiberal aims discovers a workable within-form strategy rises sharply. 
Second, efforts to make official decisions easier to review can eventually make those decisions easier to game. 
Early codification improves documentation and deters obvious abuse, but once those gains begin to level off, further standardization can make strategic adaptation easier.
Third, temporary modernization pressure can leave durable vulnerability behind.
Later governments usually inherit administrative procedures they did not design, and they may rationally stop short of fully dismantling them even after the crisis or performance pressure that justified expansion has passed.

These results change what it means to evaluate administrative reform for democratic robustness.
It is not enough to ask whether a system is strict, transparent, or well documented.
The more difficult question is whether reform makes decisions more reviewable to overseers while also making the process more legible to insiders who can repeatedly act within it. 
The danger is not auditability \emph{per se}, but reform packages that raise oversight-facing reviewability and insider-facing legibility together. 
That is a problem of authority and institutional design as much as of technology. 
Robust oversight therefore depends not only on records and review, but also on limits on who can learn from internal feedback, on meaningful separation between evaluation and deployment, and on remedies that can reverse contested actions before they persist long enough to matter.

The argument yields three empirical predictions that distinguish it from generic claims that more oversight is always protective or that automation is automatically dangerous. 
First, following political turnover, procedurally successful yet publicly contested actions should be more common where AI-governed workflows have broader routinized scope, more stable review criteria, and more reusable compliance artifacts. 
Second, this effect should be strongest where insiders can iterate rapidly on internal feedback and where evaluation is tightly coupled to deployment; it should be weaker where AI remains advisory, access to review feedback is restricted, or review is institutionally decoupled from implementation. 
Third, crisis- or capacity-driven expansions should be followed by selective rather than complete rollback. 
Organizations should preserve auditability while only partially reducing insider-facing standardization, especially where installed scale is large. 
Where internal records are available, one should observe traces of probing and selective repair: repeated reframing around approval criteria, strategic sequencing, segmentation without full decoupling, and unequal rollback across units with different inherited installed bases. 
These signatures are distinctive because the relevant interaction is between turnover, routinized scope, and access to iterative feedback, not automation generally. 

Theoretically, the next step is to model how safeguards are built, weakened, and repaired over time, and how within-form erosion interacts with more overt abuses of power.
The broader implication is not that auditability or procedural regularity is undesirable. 
It is that expansions in AI use should be judged not only by their immediate administrative gains, but by the larger institutional packages they create and how those packages perform under partisan turnover. 
Changes that improve oversight visibility can be protective, but changes that simultaneously increase insider iterability, transferability, and persistence can create a learnable approval boundary for future leaders. 
Democracies should therefore ask how AI-mediated administration, even when paired with oversight procedures, may create durable institutional capacities that political successors inherit, learn, and exploit.

\bibliographystyle{plainnat}
\bibliography{ai_governance}

\clearpage

\appendix

\section{Appendix overview and notation}
This online appendix is organized in four parts.
The current section reviews the notation. 
Section B collects proofs and derivations used directly in the Results section.
Sections C--D provide microfoundations and interpretive support for the within-form mechanism.
Sections E--G collect the codification and persistence extensions.
Sections H--J record robustness notes and a final illustration.

For quick reference, the main objects are
\[
\begin{aligned}
x &:\ \text{scale, the share of official decisions routed through AI-assisted review},\\
s &:\ \text{codification, the degree of standardization and repeatability},\\
r=(r_m,r_\kappa,r_q) &:\ \text{safeguards for unilateral throughput, test-to-deploy coupling, and remedy speed},\\
p_{\mathrm{wf}}(x,s;r) &:\ \text{within-form success probability under an autocratic successor},\\
F_0(x,s,r) &:\ \text{overt-abuse success probability},\\
F(x,s,r) &:\ \text{total democratic failure under the benchmark aggregation},\\
S(x) &:\ \text{minimum codification needed to operate the workflow at scale } x,\\
U_1(x,s;r,\lambda) &:\ \text{adoption-stage objective trading administrative value against turnover fragility}.
\end{aligned}
\]

The Results section uses both a search-based representation,
\[
\mu^{\mathrm{search}}(x,s,r)\equiv \mu_0(r)+N(x)\nu(s;r),
\qquad
p_{\mathrm{wf}}(x,s;r)=1-\exp\!\big(-\mu^{\mathrm{search}}(x,s,r)\big),
\]
and the Poisson benchmark,
\[
\mu(x,s,r)\equiv \mu_0(r)+\eta(r)\,x\,s,
\qquad
p_{\mathrm{wf}}(x,s;r)=1-\exp\!\big(-\mu(x,s,r)\big),
\]
with safeguards shifting $\mu_0(r)$ and $\eta(r)$ in protective directions.

Under the separable aggregation used in the main text,
\[
F(x,s,r)=p_{\mathrm{wf}}(x,s;r)+\big(1-p_{\mathrm{wf}}(x,s;r)\big)F_0(x,s,r)
=1-\big(1-F_0(x,s,r)\big)\exp\!\big(-\mu(x,s,r)\big).
\]
The adoption stage imposes the operational codification constraint $s\ge S(x)$ and uses
\[
U_1(x,s;r,\lambda)=\lambda\,G(x,s)-C(x,s)-\delta\,\Omega\,F(x,s,r),
\]
with $\lambda$ denoting modernization pressure, $\delta$ turnover risk, and $\Omega$ the loss from democratic failure.

\section{Proofs and derivations}
\label{app:proofs}
This section provides full proofs and supporting derivatives referenced in the Results section.
Subsection B.1 derives the threshold characterization for alignment-surface exploitability.
Subsection B.2 gives the scale comparative statics, both holding codification fixed and along the binding regime.
Subsection B.3 records the leverage and dominance derivatives under the benchmark aggregation.
Subsection B.4 gives the binding-codification argument and the threshold-crossing proof for modernization pressure.

\subsection{Proof of the threshold characterization for alignment-surface exploitability}
Proposition \ref{prop:alignment_surface_scale} expresses within-form success in single-index form.
This subsection derives the threshold characterization used in the main text for when the inherited review boundary becomes politically exploitable.

Within-form success is
\[
p_{\mathrm{wf}}(x,s;r)=1-\exp\!\big(-\mu^{\mathrm{search}}(x,s,r)\big),
\qquad
\mu^{\mathrm{search}}(x,s,r)\equiv \mu_0(r)+N(x)\nu(s;r),
\]
where $\nu(s;r)\equiv-\ln(1-\pi(s;r))$ and $\pi(s;r)$ is defined in \eqref{eq:pi_def}.

Differentiate with respect to $x$:
\[
\frac{\partial p_{\mathrm{wf}}}{\partial x}
=
\exp\!\big(-\mu^{\mathrm{search}}(x,s,r)\big)\,\frac{\partial \mu^{\mathrm{search}}}{\partial x}
=
\exp\!\big(-\mu^{\mathrm{search}}(x,s,r)\big)\,N'(x)\nu(s;r),
\]
which is weakly nonnegative when $N'(x)\ge 0$ and $\nu(s;r)\ge 0$. Differentiate with respect to $s$:
\[
\frac{\partial p_{\mathrm{wf}}}{\partial s}
=
\exp\!\big(-\mu^{\mathrm{search}}(x,s,r)\big)\,\frac{\partial \mu^{\mathrm{search}}}{\partial s}
=
\exp\!\big(-\mu^{\mathrm{search}}(x,s,r)\big)\,N(x)\nu_s(s;r),
\]
which is weakly nonnegative when $\nu_s(s;r)\ge 0$ and $N(x)\ge 0$. Differentiate with respect to a safeguard component $r_j$:
\[
\frac{\partial p_{\mathrm{wf}}}{\partial r_j}
=
\exp\!\big(-\mu^{\mathrm{search}}(x,s,r)\big)\Big(\mu_{0,r_j}(r)+N(x)\nu_{r_j}(s;r)\Big),
\]
which is weakly nonpositive when $\mu_{0,r_j}(r)\le 0$ and $\nu_{r_j}(s;r)\le 0$.

For the threshold associated with alignment-surface exploitability, fix $\bar p\in(0,1)$ and impose $p_{\mathrm{wf}}(x,s;r)\le \bar p$:
\[
1-\exp\!\big(-\mu^{\mathrm{search}}(x,s,r)\big)\le \bar p
\quad\Longleftrightarrow\quad
\exp\!\big(-\mu^{\mathrm{search}}(x,s,r)\big)\ge 1-\bar p.
\]
Taking logs and multiplying by $-1$ yields the equivalent condition
\[
\mu^{\mathrm{search}}(x,s,r)\le -\ln(1-\bar p),
\]
which is \eqref{eq:alignment_surface_condition}.
This inequality does not define the review boundary itself.
It identifies when that inherited boundary is associated with non-negligible within-form success.
Under the benchmark parameterization $\mu(x,s,r)=\mu_0(r)+\eta(r)xs$, solving for $x$ yields \eqref{eq:xcrit_scale} whenever $\eta(r)s>0$.
\hfill $\square$

\begin{figure}[t]
\centering
\includegraphics[width=0.74\textwidth]{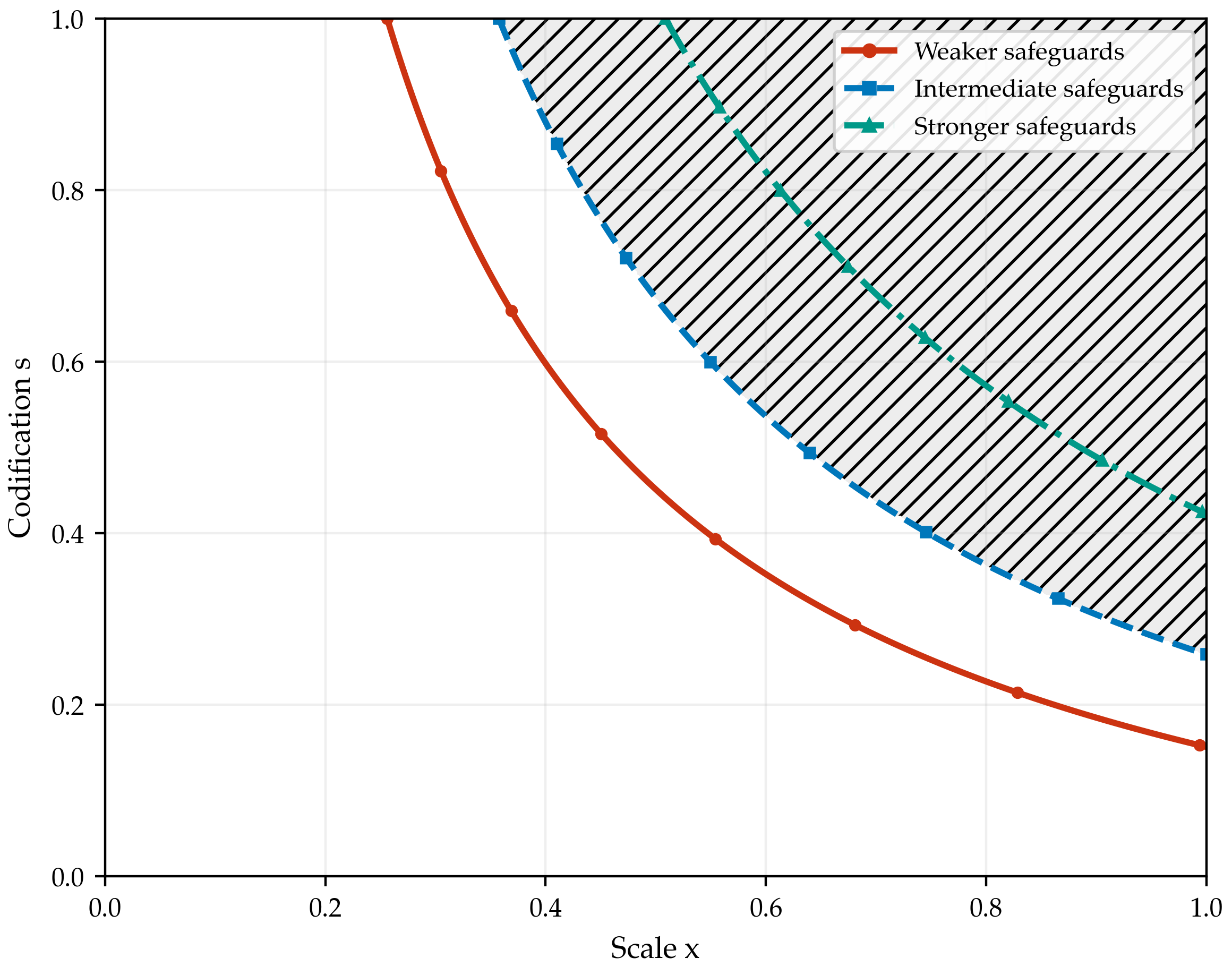}
\caption{Search-based threshold for alignment-surface exploitability.}
\label{fig:appendix_search_surface_threshold}
\floatnote{Horizontal axis: scale $x$.
Vertical axis: codification $s$.
Each curve gives the threshold condition at which within-form success equals $\bar p=0.60$, i.e., $\mu_0(r)+N(x)\nu(s;r)=\tau(\bar p)$.
Weaker safeguards shift the threshold inward, while stronger safeguards shift it outward.
Points above a curve correspond to within-form success above $\bar p$.
Schematic geometric companion to Proposition \ref{prop:alignment_surface_scale}.}
\end{figure}

\subsection{Comparative statics of total failure with respect to scale}
The Results section emphasizes that scale can reduce overt vulnerability while still increasing total failure by expanding the scope for procedurally passing failures. This subsection derives the derivative of total failure with respect to scale, first holding codification fixed and then along the binding codification regime $s=S(x)$.

\smallskip
\noindent
Write
\[
F(x,s,r)=1-\big(1-F_0(x,s,r)\big)\exp\!\big(-\mu_0(r)-\eta(r)xs\big).
\]
Holding $(s,r)$ fixed and differentiating with respect to $x$ yields
\[
\frac{\partial F}{\partial x}
=
-\Big[-F_{0,x}(x,s,r)\exp\!\big(-\mu_0(r)-\eta(r)xs\big)
+
\big(1-F_0(x,s,r)\big)\exp\!\big(-\mu_0(r)-\eta(r)xs\big)\big(-\eta(r)s\big)\Big].
\]
Rearranging gives
\begin{equation}
\frac{\partial F}{\partial x}
=
\exp\!\big(-\mu_0(r)-\eta(r)xs\big)\,F_{0,x}(x,s,r)
+
\big(1-F_0(x,s,r)\big)\exp\!\big(-\mu_0(r)-\eta(r)xs\big)\,\eta(r)s.
\label{eq:dFdx}
\end{equation}
In particular, even if $F_{0,x}(x,s,r)<0$ locally, total failure increases with scale whenever
\begin{equation}
F_{0,x}(x,s,r)>-\big(1-F_0(x,s,r)\big)\eta(r)s.
\label{eq:scale_condition}
\end{equation}
\hfill $\square$

\smallskip
\noindent
When codification is tied to scale by the binding regime $s=S(x)$, the relevant comparative static is the total derivative $dF(x,S(x),r)/dx$.
Using \eqref{eq:failure_decomp_main} and the chain rule,
\begin{equation}
\begin{aligned}
\frac{d}{dx}F(x,S(x),r)
&=
\exp\!\big(-\mu(x,S(x),r)\big)\Big(F_{0,x}(x,S(x),r)+F_{0,s}(x,S(x),r)\,S'(x)
\\
&\qquad\qquad\qquad\qquad
+\big(1-F_0(x,S(x),r)\big)\eta(r)\big(S(x)+xS'(x)\big)\Big),
\end{aligned}
\label{eq:dFdx_binding}
\end{equation}
where $\mu(x,S(x),r)=\mu_0(r)+\eta(r)xS(x)$.
Thus, even if both scaling and induced codification reduce overt vulnerability ($F_{0,x}<0$ and $F_{0,s}\le 0$), total failure rises with scale along the binding regime whenever
\begin{equation}
F_{0,x}(x,S(x),r)+F_{0,s}(x,S(x),r)\,S'(x)
>
-\big(1-F_0(x,S(x),r)\big)\eta(r)\big(S(x)+xS'(x)\big).
\label{eq:scale_condition_binding}
\end{equation}
\hfill $\square$

\subsection{Leverage and dominance under the benchmark aggregation}
\label{app:partials_summary}
This subsection records the leverage and dominance derivatives used in the Results section. Under the benchmark aggregation, total failure is monotone in overt vulnerability $F_0$ and in the within-form intensity components $\mu_0(r)$ and $\eta(r)$.

For reference, under the benchmark aggregation \eqref{eq:failure_decomp_main} and the Poisson benchmark $p_{\mathrm{wf}}(x,s;r)=1-\exp(-\mu_0(r)-\eta(r)xs)$, the partial derivatives of $F$ with respect to $(F_0,\mu_0,\eta)$ holding $(x,s)$ fixed satisfy:
\begin{equation}
\begin{aligned}
\frac{\partial F}{\partial F_0}
&=\exp\!\big(-\mu_0(r)-\eta(r)xs\big)\ge 0,
\\
\frac{\partial F}{\partial \mu_0}
&=\big(1-F_0(x,s,r)\big)\exp\!\big(-\mu_0(r)-\eta(r)xs\big)\ge 0,
\\
\frac{\partial F}{\partial \eta}
&=\big(1-F_0(x,s,r)\big)\exp\!\big(-\mu_0(r)-\eta(r)xs\big)\,xs\ge 0.
\end{aligned}
\label{eq:partials_summary}
\end{equation}
A derivation follows.
Write
\[
F(x,s,r)=1-\big(1-F_0(x,s,r)\big)\exp\!\big(-\mu_0(r)-\eta(r)xs\big).
\]
Treating $F_0$, $\mu_0$, and $\eta$ as the arguments and holding $(x,s)$ fixed, differentiation with respect to $F_0$ yields
\[
\frac{\partial F}{\partial F_0}
=
\frac{\partial}{\partial F_0}\Big[1-\big(1-F_0\big)\exp\!\big(-\mu_0-\eta xs\big)\Big]
=
\exp\!\big(-\mu_0-\eta xs\big)
\ge 0.
\]
Differentiation with respect to $\mu_0$ yields
\[
\frac{\partial F}{\partial \mu_0}
=
\big(1-F_0(x,s,r)\big)\exp\!\big(-\mu_0-\eta xs\big)
\ge 0.
\]
Differentiation with respect to $\eta$ yields
\[
\frac{\partial F}{\partial \eta}
=
\big(1-F_0(x,s,r)\big)\exp\!\big(-\mu_0-\eta xs\big)\,xs
\ge 0,
\]
since $1-F_0(x,s,r)\ge 0$ and $xs\ge 0$. The dominance statement follows immediately: if $\mu_0(\hat r)\le \mu_0(r)$, $\eta(\hat r)\le \eta(r)$, and $F_0(x,s,\hat r)\le F_0(x,s,r)$ for all $(x,s)$, then monotonicity of $F$ in each argument implies $F(x,s,\hat r)\le F(x,s,r)$ for all $(x,s)$.
\hfill $\square$

\subsection{Binding codification and threshold-crossing proof}
The Results section often studies adoption along the binding codification regime $s=S(x)$. Lemma \ref{lem:binding_codification} gives a sufficient condition under which the adoption-stage choice binds at $s=S(x)$, and Proposition \ref{prop:scale_increases_lambda} then derives the comparative static of adopted scale with respect to modernization pressure along that regime, together with the induced threshold-crossing result.

\begin{lemma}[Binding codification under monotone nonincreasing objective]\label{lem:binding_codification}
Fix any $x$ and $(r,\lambda)$. Suppose that for all $s\in[S(x),1]$,
\begin{align}
\lambda\, G_s(x,s) - C_s(x,s) &\le 0, \label{eq:lem_bind_cond_1}\\
F_s(x,s,r) &\ge 0. \label{eq:lem_bind_cond_2}
\end{align}
Then $s=S(x)$ is optimal for the codification choice, i.e.
\[
S(x)\in \arg\max_{s\in[S(x),1]} U_1(x,s;r,\lambda).
\]
\end{lemma}

\smallskip
\noindent Proof.
Recall that
\[
U_1(x,s;r,\lambda)= \lambda G(x,s) - C(x,s) - \delta \Omega F(x,s,r),
\]
so
\[
U_{1,s}(x,s;r,\lambda)=\lambda G_s(x,s)-C_s(x,s)-\delta\Omega F_s(x,s,r).
\]
Under \eqref{eq:lem_bind_cond_1}--\eqref{eq:lem_bind_cond_2} we have $U_{1,s}(x,s;r,\lambda)\le 0$ for all $s\in[S(x),1]$. Hence $U_1(x,s;r,\lambda)$ is weakly decreasing in $s$ on $[S(x),1]$, so the maximum on this interval is attained at the lower endpoint. In particular, $s=S(x)$ is optimal.

\smallskip
\noindent Proof of Proposition \ref{prop:scale_increases_lambda}.
Fix $r$ and work along the binding codification regime $s=S(x)$.

\smallskip
\noindent
Part (i).
Define
\[
U_1(x,S(x);r,\lambda)=\lambda\,G(x,S(x))-C(x,S(x))-\delta\,\Omega\,F(x,S(x),r).
\]
By assumption, an interior maximizer $x^*(\lambda;r)$ exists and $U_1(\cdot,S(\cdot);r,\lambda)$ is strictly concave in $x$ over the relevant range, so the maximizer is unique and satisfies the first-order condition
\[
\frac{\partial U_1}{\partial x}(x^*(\lambda;r),S(x^*(\lambda;r));r,\lambda)=0.
\]
Differentiate this condition with respect to $\lambda$. By the implicit function theorem, the derivatives satisfy
\[
\frac{\partial^2 U_1}{\partial x^2}(x^*(\lambda;r),S(x^*(\lambda;r));r,\lambda)\cdot \frac{d x^*(\lambda;r)}{d\lambda}
+
\frac{\partial^2 U_1}{\partial x\,\partial \lambda}(x^*(\lambda;r),S(x^*(\lambda;r));r,\lambda)
=0.
\]
Because $\lambda$ enters only through the term $\lambda\,G(x,S(x))$, we have
\[
\frac{\partial^2 U_1}{\partial x\,\partial \lambda}(x,S(x);r,\lambda)=\frac{d}{dx}G(x,S(x))>0.
\]
Here $\frac{d}{dx}G(x,S(x))=G_x(x,S(x))+G_s(x,S(x))S'(x)$, which is strictly positive under the maintained assumptions $G_x>0$, $G_s\ge 0$, and $S'(x)\ge 0$.
Strict concavity implies $\partial^2 U_1/\partial x^2<0$ at the interior optimum. Therefore,
\[
\frac{d x^*(\lambda;r)}{d\lambda}
=
-\frac{\frac{d}{dx}G(x^*(\lambda;r),S(x^*(\lambda;r)))}{\frac{\partial^2 U_1}{\partial x^2}(x^*(\lambda;r),S(x^*(\lambda;r));r,\lambda)}
>0,
\]
which establishes that $x^*(\lambda;r)$ is strictly increasing in $\lambda$ over the region where the interior solution applies.

\smallskip
\noindent
Part (ii).
Let $\tilde\mu(x;r)\equiv \mu_0(r)+\eta(r)\,xS(x)$, so that
\[
\zeta(\lambda)\equiv p_{\mathrm{wf}}(x^*(\lambda;r),S(x^*(\lambda;r));r)
=1-\exp\!\big(-\tilde\mu(x^*(\lambda;r);r)\big).
\]
Under the stated assumptions, $\tilde\mu(\cdot;r)$ is continuous and weakly increasing on $[0,\bar x]$, and $x^*(\lambda;r)$ is continuous and increasing in $\lambda$ (from part (i)), hence $\zeta(\lambda)$ is continuous and weakly increasing in $\lambda$.

Define the critical scale as the minimal scale that attains the within-form threshold:
\[
x^{\mathrm{crit}}(\bar p;r)\;\equiv\;\inf\big\{x\in[0,\bar x] : p_{\mathrm{wf}}(x,S(x);r)\ge \bar p\big\}.
\]
By continuity of $p_{\mathrm{wf}}(x,S(x);r)$ in $x$, we have
\[
p_{\mathrm{wf}}(x,S(x);r)<\bar p\ \text{ for all } x<x^{\mathrm{crit}}(\bar p;r),
\qquad
p_{\mathrm{wf}}(x,S(x);r)\ge \bar p\ \text{ for all } x\ge x^{\mathrm{crit}}(\bar p;r),
\]
and in particular $p_{\mathrm{wf}}(x^{\mathrm{crit}}(\bar p;r),S(x^{\mathrm{crit}}(\bar p;r));r)=\bar p$.

Now suppose there exist $\lambda_L<\lambda_H$ such that
\[
x^*(\lambda_L;r)<x^{\mathrm{crit}}(\bar p;r)<x^*(\lambda_H;r).
\]
Then $\zeta(\lambda_L)<\bar p$ and $\zeta(\lambda_H)\ge \bar p$. Since $\zeta(\lambda)$ is continuous, the intermediate value theorem implies that there exists $\lambda^{\mathrm{crit}}\in(\lambda_L,\lambda_H]$ such that $\zeta(\lambda^{\mathrm{crit}})=\bar p$. Moreover, since $\zeta$ is weakly increasing, any such $\lambda^{\mathrm{crit}}$ is a cutoff in the sense that $\zeta(\lambda)\le \bar p$ for $\lambda\le \lambda^{\mathrm{crit}}$ and $\zeta(\lambda)\ge \bar p$ for $\lambda\ge \lambda^{\mathrm{crit}}$.
The exact expression for total failure at $\lambda^{\mathrm{crit}}$ follows from \eqref{eq:failure_decomp_main}:
\[
F=F_0+\big(1-F_0\big)\bar p.
\]
\hfill $\square$

\bigskip
\noindent\textbf{Microfoundations and Interpretive Support}
\par\smallskip

\section{Microfoundation: Within-Form Success and Remedy-as-Persistence}
\label{app:wf_microfoundation}
This section microfounds the within-form term used in the main text. It provides a Poisson benchmark for \eqref{eq:wf_prob_main} and highlights a central institutional point: remedies reduce within-form success primarily by shortening persistence windows rather than by mechanically preventing attempts.

Normalize the relevant consolidation window to length one. Consider safeguard levers $r=(r_m,r_\kappa,r_q)$. Let $M(r_m)\ge 0$ capture effective unilateral throughput and iteration capacity, let $\kappa(r_\kappa)\in[0,1]$ capture the strength of test-to-deployment coupling, and let $q(r_q)>0$ denote a remedy hazard that increases with contestation and reversal capacity. Let $\psi_0(q)\in(0,1)$ denote the probability that a baseline within-form move persists long enough to matter, and let $\psi(q)\in(0,1)$ denote the analogous persistence probability for standardized within-form moves.

We assume the safeguard levers shift these objects in the natural directions:
\[
M'(r_m)<0,\qquad \kappa'(r_\kappa)<0,\qquad q'(r_q)>0,
\qquad
\psi_0'(q)<0,\qquad \psi'(q)<0.
\]
Define a baseline intensity
\[
\mu_0(r)\equiv M(r_m)\psi_0\!\big(q(r_q)\big),
\]
and a standardized intensity coefficient
\[
\eta(r)\equiv M(r_m)\kappa(r_\kappa)\psi\!\big(q(r_q)\big).
\]
Interpretation: safeguards reduce $\mu_0(r)$ by limiting throughput and speeding remedies, and they reduce $\eta(r)$ by additionally decoupling testing from deployment.

The $xs$ structure in the standardized component has a simple institutional interpretation. Scale $x$ increases the number of routinized decisions routed through the review process, which creates more opportunities to propose and revise contested actions. Codification $s$ makes approval criteria more stable and transferable across settings, which raises the expected value of each probe and makes what is learned in one domain usable in others. Under an independence approximation across routinized decision points, an arrival process for effective standardized moves is well approximated by a Poisson process with mean proportional to $xs$.

In many settings, remedy and contestation processes have minimum time requirements (due process, evidentiary standards, and adjudicatory capacity), which can be represented as an upper bound $q(r_q)\le \bar q<\infty$ on the remedy hazard achievable within the window. Similarly, even strong separation of duties may not eliminate executive probing entirely, which can be represented as a lower bound $M(r_m)\ge \underline m>0$. Finally, complete decoupling may be infeasible because evaluation must remain meaningfully predictive of enforcement to be administrable and legitimate, which can be represented as a lower bound $\kappa(r_\kappa)\ge \underline\kappa>0$. Under these speed-limit conditions, $\psi_0(q(r_q))\ge \psi_0(\bar q)>0$ and $\psi(q(r_q))\ge \psi(\bar q)>0$, so both indices are bounded away from zero:
\[
\mu_0(r)\ge \underline m\,\psi_0(\bar q)\equiv \underline \mu_0>0,
\qquad
\eta(r)\ge \underline m\,\underline\kappa\,\psi(\bar q)\equiv \underline \eta>0.
\]

One can make the ``mean proportional to $xs$'' claim slightly less tautological with a toy iteration model. Suppose that, in the consolidation window, there are $n(x)$ contested decision points routed through a standardized review process, with $n(x)$ increasing in $x$. At each routed decision point, an autocratic executive can attempt $M(r_m)$ procedurally plausible variants. Let $\rho(s,r_\kappa)\in(0,1)$ denote the probability that a proposed variant both satisfies the approval criteria and is substantively erosive, and let $\psi\!\big(q(r_q)\big)\in(0,1)$ denote the probability that a passing erosive move persists long enough to matter. Then the per-attempt effective probability is $\rho(s,r_\kappa)\psi\!\big(q(r_q)\big)$, and the number of effective standardized moves is binomial with mean $n(x)M(r_m)\rho(s,r_\kappa)\psi\!\big(q(r_q)\big)$. In regimes with many attempts and small per-attempt success probability, a Poisson approximation implies an effective arrival rate proportional to $n(x)M(r_m)\rho(s,r_\kappa)\psi\!\big(q(r_q)\big)$. If $n(x)$ is approximately linear in scale over the relevant range and $\rho(s,r_\kappa)$ is approximately linear in codification (e.g., $\rho(s,r_\kappa)\approx s\,\kappa(r_\kappa)$ as a local first-order approximation), then this yields a mean of the form $\eta(r)xs$ as a tractable reduced form.

Let the number of effective baseline moves be $N_0\sim\mathrm{Poisson}(\mu_0(r))$. Let the number of effective standardized moves be $N_{\mathrm{std}}\sim\mathrm{Poisson}(\eta(r)xs)$, reflecting that scale creates more opportunities for procedurally passing moves while codification makes successful tactics easier to reuse. If $N_0$ and $N_{\mathrm{std}}$ are independent, then $N\equiv N_0+N_{\mathrm{std}}\sim\mathrm{Poisson}(\mu_0(r)+\eta(r)xs)=\mathrm{Poisson}(\mu(x,s,r))$. Therefore,
\[
p_{\mathrm{wf}}(x,s;r)=\Pr[N\ge 1]=1-\exp\!\big(-\mu(x,s,r)\big),
\]
which matches \eqref{eq:wf_prob_main}.

A cumulative notion of erosion can be incorporated without changing the monotone logic. If regime-relevant within-form erosion requires at least $k\ge 1$ effective moves within the window, then with $N\sim\mathrm{Poisson}(\mu(x,s,r))$,
\[
p_{\mathrm{wf}}^{(k)}(x,s;r)=\Pr[N\ge k]
=
1-\exp\!\big(-\mu(x,s,r)\big)\sum_{j=0}^{k-1}\frac{\big(\mu(x,s,r)\big)^j}{j!}.
\]
This variant remains increasing in $x$ and decreasing in safeguards whenever $\mu_{0,r_j}(r)\le 0$ and $\eta_{r_j}(r)\le 0$.
For any target threshold $\bar p\in(0,1)$, define the corresponding intensity cutoff $\tau_k(\bar p)$ as the smallest $\mu$ such that $p_{\mathrm{wf}}^{(k)}(\mu)\ge \bar p$.
Then the threshold characterization for alignment-surface exploitability continues to hold with the same index $\mu(x,s,r)=\mu_0(r)+\eta(r)xs$, replacing the Poisson cutoff $-\ln(1-\bar p)$ by $\tau_k(\bar p)$.

\section{Interpretive Note: Asymmetric Access, Boundary Learning, and Ambiguity Reduction}
\label{app:alt_microfoundation}
This section is an interpretive clarification rather than a new core result.
The main text derives within-form success from a stylized search environment. This short note adds an institutional interpretation and records how ambiguity reduction can be incorporated without changing the basic mechanism.

The key asymmetry is organizational rather than computational. In many settings, the executive (or political deployer) participates directly in proposal-and-revision against the review process, while oversight bodies operate through distinct mandates, staffing constraints, and procedures. As a result, a standardization investment can simultaneously increase oversight-facing auditability and increase deployer-facing learnability, because these are different channels. In the model, this asymmetry is captured by the fact that the executive's effective per-attempt probability of discovering an erosive-but-passing move, $\rho(s,r_\kappa)$, and the number of attempts it can run, $M(r_m)$, need not be available to external oversight at comparable speed or resolution. Appendix \ref{app:codification_split} formalizes the wedge by separating insider-facing standardization from oversight-facing auditability.

Codification also can reduce within-form vulnerability by shrinking the effective ambiguity set, for example by clarifying doctrine or reducing discretionary degrees of freedom. The framework is compatible with this possibility. Replace the standardized component of within-form intensity with
\[
\mu(x,s,r)=\mu_0(r)+\eta(r)\,x\,s\cdot \omega(s),
\]
where $\omega(s)\in(0,1]$ captures the effective mass of contested cases that remain exploitable after operationalization, with $\omega'(s)\le 0$. Then codification raises within-form vulnerability whenever the learnability term grows faster than ambiguity shrinks, i.e., whenever $\frac{d}{ds}[s\,\omega(s)]\ge 0$ over the relevant range. The main text treats $\omega(s)$ as approximately constant because administrative procedures can stabilize approval criteria without settling underlying constitutional contestation; this note records how to incorporate ambiguity reduction if desired.

\bigskip
\noindent\textbf{Codification and Persistence Extensions}
\par\smallskip

\section{Codification Result: Cutoff, Proof, and Illustration}
\label{app:robustness_codification}
This section provides supporting codification results.
It first states the benchmark cutoff for when marginal codification increases total failure.
It then gives the proof and a minimal illustration for Proposition \ref{prop:codification_flip_main} from the main text.

Let codification be $s\in[0,1]$, and let within-form success be
\[
p_{\mathrm{wf}}(x,s;r)=1-\exp\!\big(-\mu_0(r)-\eta(r)xs\big).
\]
Let the overt-abuse component be $F_0(x,s,r)$ with $F_{0,s}(x,s,r)<0$ and $F_{0,r_j}(x,s,r)\le 0$ for each safeguard component $r_j$, reflecting that codified compliance deters clearly impermissible actions and that safeguards strengthen enforcement and remedies. The sign of $F_{0,x}$ is not restricted and is irrelevant for the codification results in this section. Total failure under an autocrat is
\[
F(x,s,r)=1-\big(1-F_0(x,s,r)\big)\exp\!\big(-\mu_0(r)-\eta(r)xs\big).
\]

To motivate the one-crossing flip, suppose codification deters overt abuse with diminishing marginal returns. As a reduced-form microfoundation, let an overtly impermissible attempt succeed only if it is not detected and enjoined within the relevant window. Let codification raise a detection/enforcement hazard $d(s)$ with $d'(s)>0$ and $d''(s)\le 0$. If the survival probability of an overtly impermissible attempt is $\exp(-d(s))$, then for some baseline $\bar F(x,r)\in(0,1)$ we can write $F_0(x,s,r)=\bar F(x,r)\exp(-d(s))$. This implies $F_{0,s}\le 0$ and $F_{0,ss}\ge 0$ whenever $d'(s)>0$ and $d''(s)\le 0$, since $F_{0,ss}=\bar F\exp(-d)\big(d'(s)^2-d''(s)\big)$.

\begin{proposition}[Codification tradeoff and a structural cutoff]
\label{prop:codification_cutoff_appendix}
Holding $(x,r)$ fixed,
\[
\frac{\partial F}{\partial s}
=
\exp\!\big(-\mu_0(r)-\eta(r)xs\big)\Big(F_{0,s}(x,s,r)+\big(1-F_0(x,s,r)\big)\eta(r)x\Big).
\]
Thus codification increases failure risk locally if and only if
\[
\big(1-F_0(x,s,r)\big)\eta(r)x > -F_{0,s}(x,s,r).
\]
\end{proposition}

Proof of Proposition \ref{prop:codification_cutoff_appendix}.
Differentiate
\[
F(x,s,r)=1-\big(1-F_0(x,s,r)\big)\exp\!\big(-\mu_0(r)-\eta(r)xs\big)
\]
with respect to $s$ holding $(x,r)$ fixed:
\[
\frac{\partial F}{\partial s}
=
-\Big[-F_{0,s}\exp\!\big(-\mu_0-\eta xs\big)
+
\big(1-F_0\big)\exp\!\big(-\mu_0-\eta xs\big)\big(-\eta x\big)\Big].
\]
Rearranging yields
\[
\frac{\partial F}{\partial s}
=
\exp\!\big(-\mu_0-\eta xs\big)\Big(F_{0,s}+\big(1-F_0\big)\eta x\Big),
\]
and the exponential factor is strictly positive, so the sign condition follows.
\hfill $\square$

Proof of Proposition \ref{prop:codification_flip_main}.
From Proposition \ref{prop:codification_cutoff_appendix},
\[
\frac{\partial F}{\partial s}
=
\exp\!\big(-\mu_0(r)-\eta(r)xs\big)\,h(s),
\]
and the exponential term is strictly positive, so the sign of $\partial F/\partial s$ equals the sign of $h(s)$. Differentiate $h(s)$:
\[
h'(s)=F_{0,ss}(x,s,r)-\eta(r)x\,F_{0,s}(x,s,r).
\]
If $F_{0,ss}(x,s,r)\ge 0$ and $F_{0,s}(x,s,r)\le 0$, then $h'(s)\ge 0$, so $h(s)$ is weakly increasing and crosses zero at most once. If $h(0)<0<h(1)$, continuity implies existence of at least one root in $(0,1)$. If, in addition, $x>0$ and $F_{0,s}(x,s,r)<0$ for interior $s$, then $h'(s)>0$ and the root $s^{\mathrm{flip}}\in(0,1)$ is unique, yielding the sign characterization.
\hfill $\square$

\smallskip
\noindent
The cutoff has a direct institutional interpretation. Codification reduces overt risk through $F_{0,s}<0$, but it can also increase within-form risk by making approval criteria stable enough for legal iteration to be productive. The condition is therefore more likely to reverse when scale is high, when safeguards are weak (so $\eta(r)$ is large), and when the marginal overt-deterrence gain from codification is small.

\smallskip
\noindent
Minimal illustration.
A simple overt-abuse channel that satisfies $F_{0,s}\le 0$ and $F_{0,ss}\ge 0$ is
\[
F_0(x,s,r)=\bar F(x,r)\,e^{-bs},
\qquad
b>0,
\qquad
\bar F(x,r)\in(0,1).
\]
In this case,
\[
h(s)=\eta(r)x-(\eta(r)x+b)\bar F(x,r)e^{-bs},
\]
so the unique flip point (when it lies in $(0,1)$) is
\[
s^{\mathrm{flip}}(x,r)=\frac{1}{b}\ln\!\left(\frac{(\eta(r)x+b)\bar F(x,r)}{\eta(r)x}\right).
\]

\section{Codification Split: Separating Standardization from Auditability}
\label{app:codification_split}
The main text uses a single codification index $s$ to capture the empirically common case in which administrative modernization raises both oversight-facing auditability and insider-facing boundary stability through the same investments in standardization, documentation, logging, and evaluation. This section unbundles that adoption-stage reform margin into separate components. The point is to show which part of the bundle drives each effect in the benchmark and to supply the notation used in the narrower post-crisis repair extension in Section~\ref{sec:results}.

The split between $s_{\mathrm{std}}$ and $s_{\mathrm{aud}}$ can also be microfounded by the need to operationalize probabilistic outputs.
Suppose each action $a$ generates a compliance score $z(a)\in[0,1]$.
The institution chooses an operationalization protocol $(\tau_L,\tau_H,\ell)$, where $\tau_L\le \tau_H$ define the automatic Flag/Pass regions and $\ell$ indexes logging and recordkeeping depth.
Let $m\equiv \tau_H-\tau_L$ denote the width of the uncertainty band.
Interpret narrower bands as greater insider-facing standardization and richer logging as greater oversight-facing auditability:
\[
s_{\mathrm{std}} = g(1-m), \qquad s_{\mathrm{aud}} = h(\ell),
\]
with $g'>0$ and $h'>0$.
The rest of this section studies how those two margins enter total failure.

Let $s_{\mathrm{std}}\in[0,1]$ capture stability and transferability of the Pass/Flag boundary for an insider with access to feedback, and let $s_{\mathrm{aud}}\in[0,1]$ capture auditability and traceability that help oversight detect and contest overt violations. Let within-form intensity be
\[
\mu(x,s_{\mathrm{std}},r)=\mu_0(r)+\eta(r)\,x\,s_{\mathrm{std}},
\qquad
p_{\mathrm{wf}}(x,s_{\mathrm{std}};r)=1-\exp\!\big(-\mu(x,s_{\mathrm{std}},r)\big),
\]
and let overt vulnerability depend only on auditability,
\[
F_0=F_0(x,s_{\mathrm{aud}},r),
\qquad
F_{0,s_{\mathrm{aud}}}\le 0.
\]
Total failure is
\[
F(x,s_{\mathrm{std}},s_{\mathrm{aud}},r)
=
p_{\mathrm{wf}}(x,s_{\mathrm{std}};r)+\big(1-p_{\mathrm{wf}}(x,s_{\mathrm{std}};r)\big)F_0(x,s_{\mathrm{aud}},r).
\]
Then the two codification components have cleanly signed marginal effects:
\[
\frac{\partial F}{\partial s_{\mathrm{aud}}}
=
\big(1-p_{\mathrm{wf}}(x,s_{\mathrm{std}};r)\big)\,F_{0,s_{\mathrm{aud}}}(x,s_{\mathrm{aud}},r)\le 0,
\]
and, under the Poisson benchmark,
\[
\frac{\partial F}{\partial s_{\mathrm{std}}}
=
\big(1-F_0(x,s_{\mathrm{aud}},r)\big)\exp\!\big(-\mu(x,s_{\mathrm{std}},r)\big)\,\eta(r)x\ge 0,
\]
with strict inequality when $x>0$, $\eta(r)>0$, and $F_0<1$. In this split formulation there is no flip: raising auditability reduces failure risk, while raising insider-facing standardization increases failure risk.

Proposition \ref{prop:codification_flip_main} can be read as a reduced-form statement about a single investment margin that jointly raises both components. Formally, suppose a scalar codification investment $z$ induces $s_{\mathrm{aud}}=h_{\mathrm{aud}}(z)$ and $s_{\mathrm{std}}=h_{\mathrm{std}}(z)$ with $h'_{\mathrm{aud}},h'_{\mathrm{std}}>0$. Then
\[
\frac{d}{dz}F(x,h_{\mathrm{std}}(z),h_{\mathrm{aud}}(z),r)
=
\frac{\partial F}{\partial s_{\mathrm{aud}}}h'_{\mathrm{aud}}(z)
+
\frac{\partial F}{\partial s_{\mathrm{std}}}h'_{\mathrm{std}}(z).
\]
If the marginal enforcement benefit $-(\partial F/\partial s_{\mathrm{aud}})$ declines with $z$ while the marginal within-form cost $(\partial F/\partial s_{\mathrm{std}})$ is approximately constant or rises over the relevant range, then the total derivative can change sign once, yielding the one-crossing double-edged pattern in $z$. This split also motivates an empirically and institutionally important wedge emphasized in the main text: reforms that increase auditability for oversight without increasing deployer-facing iterability and transferability are unambiguously protective in this benchmark.

\section{Persistence Extension: Temporary Pressure and Incomplete Post-Crisis Unwinding}
\label{app:post_crisis_unwinding}

This section is a downstream extension built on the split notation from Section~\ref{app:codification_split}.
It studies the case in which installed scale and auditability are inherited from a temporary high-pressure episode, while post-crisis repair can operate only by reducing insider-facing standardization.
The key question is whether a democratic government unwinds enough of that inherited standardization to return below the within-form concern threshold.
Proposition \ref{prop:post_crisis_unwinding} shows that repair can be positive yet incomplete.

Fix an installed scale $x>0$, inherited safeguards $r$, and a concern threshold $\bar p\in(0,1)$. 
Define the critical level of insider-facing standardization by
\[
s_{\mathrm{std}}^{\mathrm{crit}}(x;r,\bar p)
\equiv
\left[
\frac{-\ln(1-\bar p)-\mu_0(r)}{\eta(r)x}
\right]_{[0,1]},
\]
where $[\cdot]_{[0,1]}$ truncates to the interval $[0,1]$. 
Then
\[
p_{\mathrm{wf}}(x,s_{\mathrm{std}};r)\le \bar p
\quad \Longleftrightarrow \quad
s_{\mathrm{std}}\le s_{\mathrm{std}}^{\mathrm{crit}}(x;r,\bar p).
\]
If
\[
\mu_0(r)\ge -\ln(1-\bar p),
\]
baseline within-form vulnerability already exceeds the threshold, so lowering insider-facing standardization cannot by itself return the polity to the low-risk region.

Suppose a temporary high-pressure episode leaves inherited state $(x_H,s_{\mathrm{aud},H},s_{\mathrm{std},H})$. 
Let the amount of unwinding needed to return below the threshold be
\[
g_H
\equiv
\left(s_{\mathrm{std},H}-s_{\mathrm{std}}^{\mathrm{crit}}(x_H;r,\bar p)\right)_+,
\]
where $(\cdot)_+$ denotes the positive part. 
When $g_H>0$, the inherited crisis architecture lies above the within-form threshold.

After pressure subsides to $\lambda_L$, let
\[
B(x,s_{\mathrm{aud}},s_{\mathrm{std}};\lambda)
\equiv
\lambda G(x,s_{\mathrm{aud}},s_{\mathrm{std}})-C(x,s_{\mathrm{aud}},s_{\mathrm{std}})
\]
denote ordinary-governance value net of operating cost under the split notation. 
The post-crisis democratic government chooses refactoring effort $u\in[0,s_{\mathrm{std},H}]$, where post-repair insider-facing standardization is $s_{\mathrm{std},1}=s_{\mathrm{std},H}-u$. 
Its problem is
\[
\max_{u\in[0,s_{\mathrm{std},H}]}
W(u)
=
B(x_H,s_{\mathrm{aud},H},s_{\mathrm{std},H}-u;\lambda_L)
-
K(u;x_H)
-
\delta \Omega F(x_H,s_{\mathrm{std},H}-u,s_{\mathrm{aud},H},r).
\]
Suppose refactoring costs satisfy
\[
K_u(u;x)>0,\qquad K_{uu}(u;x)\ge 0,\qquad K_{ux}(u;x)>0.
\]
These conditions say that repair is costly, marginal repair costs are weakly increasing, and larger inherited systems are harder to decouple.

Define the marginal benefit of lowering insider-facing standardization by
\[
\Delta(x,s_{\mathrm{aud}},s_{\mathrm{std}},r;\lambda)
\equiv
\delta \Omega F_{s_{\mathrm{std}}}(x,s_{\mathrm{std}},s_{\mathrm{aud}},r)
-
B_{s_{\mathrm{std}}}(x,s_{\mathrm{aud}},s_{\mathrm{std}};\lambda).
\]
Then
\[
W_u(u)
=
\Delta(x_H,s_{\mathrm{aud},H},s_{\mathrm{std},H}-u,r;\lambda_L)-K_u(u;x_H).
\]
Under the split benchmark,
\[
F_{s_{\mathrm{std}}}(x,s_{\mathrm{std}},s_{\mathrm{aud}},r)
=
\big(1-F_0(x,s_{\mathrm{aud}},r)\big)
\exp\!\big(-\mu_0(r)-\eta(r)x s_{\mathrm{std}}\big)\eta(r)x,
\]
so the political value of refactoring is shaped by installed scale, but it falls as insider-facing standardization is unwound.

\begin{proposition}[Temporary pressure and incomplete post-crisis unwinding]
\label{prop:post_crisis_unwinding}
Fix inherited safeguards $r$, a low-pressure environment $\lambda_L$, and a concern threshold $\bar p\in(0,1)$. 
Suppose a temporary high-pressure episode leaves inherited state $(x_H,s_{\mathrm{aud},H},s_{\mathrm{std},H})$ with
\[
s_{\mathrm{std},H}>s_{\mathrm{std}}^{\mathrm{crit}}(x_H;r,\bar p),
\]
so that $g_H>0$. 
Suppose also that
\[
\Delta(x_H,s_{\mathrm{aud},H},s_{\mathrm{std},H},r;\lambda_L)>K_u(0;x_H)
\]
and
\[
K_u(g_H;x_H)>
\sup_{s_{\mathrm{std}}\in[0,s_{\mathrm{std}}^{\mathrm{crit}}(x_H;r,\bar p)]}
\Delta(x_H,s_{\mathrm{aud},H},s_{\mathrm{std}},r;\lambda_L).
\]
Then every optimizer $u^*$ of the repair problem satisfies
\[
0<u^*<g_H.
\]
Hence post-crisis repair is positive but incomplete:
\[
s_{\mathrm{std},H}-u^*>s_{\mathrm{std}}^{\mathrm{crit}}(x_H;r,\bar p),
\]
and therefore
\[
p_{\mathrm{wf}}(x_H,s_{\mathrm{std},H}-u^*;r)>\bar p.
\]
\end{proposition}

\noindent Proof.
The first inequality implies $W_u(0)>0$, so zero repair cannot be optimal. 
Now fix any $u\ge g_H$. 
Then $s_{\mathrm{std},H}-u\in[0,s_{\mathrm{std}}^{\mathrm{crit}}(x_H;r,\bar p)]$, so
\[
\Delta(x_H,s_{\mathrm{aud},H},s_{\mathrm{std},H}-u,r;\lambda_L)
\le
\sup_{s_{\mathrm{std}}\in[0,s_{\mathrm{std}}^{\mathrm{crit}}(x_H;r,\bar p)]}
\Delta(x_H,s_{\mathrm{aud},H},s_{\mathrm{std}},r;\lambda_L).
\]
Because $K_u(\cdot;x_H)$ is weakly increasing,
\[
K_u(u;x_H)\ge K_u(g_H;x_H)
\qquad \text{for all } u\ge g_H.
\]
Combining these inequalities with the second displayed condition implies $W_u(u)<0$ for all $u\ge g_H$. 
Therefore $W(u)$ is strictly decreasing on $[g_H,s_{\mathrm{std},H}]$, so no optimizer can lie at or beyond $g_H$. 
Since zero repair is not optimal, every optimizer must satisfy $0<u^*<g_H$. 
The threshold claim follows immediately from the definition of $g_H$.
\hfill $\square$

A convenient special case is quadratic refactoring cost:
\[
K(u;x)=\frac{\kappa}{2}(1+\phi x)u^2,
\qquad
\kappa,\phi>0.
\]
Then $K_u(u;x)=\kappa(1+\phi x)u$, so larger inherited systems are mechanically harder to unwind at every point. 
Holding the distance to the threshold fixed, incomplete unwinding becomes easier to sustain as installed scale grows.

\bigskip
\noindent\textbf{Robustness and Illustration}
\par\smallskip

\section{Robustness: Alternative aggregation and strategic choice}
\label{app:alt_aggregation}
The benchmark total-failure form $F=p_{\mathrm{wf}}+(1-p_{\mathrm{wf}})F_0$ can be read as the probability that at least one of two failure channels succeeds within the consolidation window. This is natural when an autocrat can attempt both within-form and overt moves, and failure occurs if either channel produces regime-relevant democratic breakdown. It can also be interpreted as a tractable separable benchmark that isolates the paper's main mechanism: how the architecture shifts within-form success $p_{\mathrm{wf}}$ and overt success $F_0$.

If instead an autocrat chooses a channel strategically, a simple alternative benchmark is
\[
F^{\max}(x,s,r)\equiv \max\!\big\{F_0(x,s,r),\,p_{\mathrm{wf}}(x,s;r)\big\}.
\]
This choice-based formulation says the autocrat concentrates effort on the more promising channel rather than pursuing both in parallel. It therefore differs from the separable benchmark by ruling out the extra probability mass that comes from attempting both channels within the same window.

The qualitative comparative statics used in the main text continue to hold under $F^{\max}$. The aggregator $F^{\max}$ is weakly increasing in each channel holding the other fixed. Therefore:
\[
\frac{\partial F^{\max}}{\partial p_{\mathrm{wf}}}\ge 0
\qquad\text{and}\qquad
\frac{\partial F^{\max}}{\partial F_0}\ge 0
\]
almost everywhere, with the only nondifferentiability at the knife-edge set where $p_{\mathrm{wf}}=F_0$. As a result, any architectural change that raises within-form success while not lowering overt success cannot reduce $F^{\max}$, and any safeguard bundle that weakly lowers both $p_{\mathrm{wf}}$ and $F_0$ weakly lowers $F^{\max}$ as well.

This is enough for the paper's directional claims.
Scale and codification remain weakly risk-increasing through the within-form channel when they raise $p_{\mathrm{wf}}$, safeguards remain weakly protective when they lower both channels, and the alignment-surface logic continues to identify a within-form vulnerability threshold independent of the exact aggregator.
What does change is the leverage of overt-only reforms: when $p_{\mathrm{wf}}>F_0$, marginal reductions in $F_0$ do not change $F^{\max}$ until overt risk again becomes the dominant channel.
The benchmark aggregator in the main text avoids this corner and is therefore better suited for smooth comparative statics, while $F^{\max}$ captures the strategic-substitution interpretation directly.

The benchmark aggregation \eqref{eq:failure_decomp_main} can also be relaxed in the complementary direction. In many political settings, within-form erosion can weaken enforcement institutions, which would make overt abuse more likely after some within-form success. A simple reduced-form way to represent this is to allow the overt-abuse component to depend on within-form intensity:
\[
F_0=F_0(x,s,r;\mu),
\qquad
\frac{\partial F_0}{\partial \mu}\ge 0,
\]
and to keep the total-failure form $F=F_0+(1-F_0)p_{\mathrm{wf}}$. Under $\partial F_0/\partial \mu\ge 0$, increases in scale or codification raise within-form intensity $\mu$ and thereby (weakly) increase overt vulnerability in addition to increasing $p_{\mathrm{wf}}$. Therefore the separable benchmark used in the main text is conservative for the comparative statics emphasized there: allowing complementarity strengthens, rather than weakens, the cautionary scale and codification implications.

\section{Robustness: Nonlinear Within-Form Intensity}
\label{app:nonlinear_mu}
Several results in the main text are not specific to the linear index $\mu(x,s,r)=\mu_0(r)+\eta(r)xs$. To see this, replace \eqref{eq:mu_def} with
\[
\mu(x,s,r)=\mu_0(r)+\eta(r)\,\phi(x,s),
\]
where $\phi(x,s)\ge 0$ is weakly increasing in both $x$ and $s$ (e.g., $\phi(x,s)=x^\alpha s^\beta$ with $\alpha,\beta>0$). Under the Poisson benchmark $p_{\mathrm{wf}}=1-\exp(-\mu)$, the within-form level set $\{(x,s,r):p_{\mathrm{wf}}(x,s;r)\le \bar p\}$ is still characterized by a single-index inequality:
\[
\mu_0(r)+\eta(r)\,\phi(x,s)\le -\ln(1-\bar p).
\]
Along the binding codification regime $s=S(x)$, the threshold-crossing logic requires only that $\phi(x,S(x))$ be increasing in $x$, so that increasing modernization pressure pushes the system monotonically toward higher within-form vulnerability. The linear specification $\phi(x,s)=xs$ is used in the main text for transparency and for the closed-form one-crossing codification flip.

The same logic extends beyond the Poisson benchmark. Suppose within-form success takes the form $p_{\mathrm{wf}}(x,s;r)=H(\mu(x,s,r))$ where $H$ is weakly increasing, and suppose the within-form intensity index takes the form $\mu(x,s,r)=\mu_0(r)+\eta(r)\phi(x,s)$ where $\phi$ is weakly increasing in both $x$ and $s$. For any target threshold $\bar p\in(0,1)$, define the corresponding intensity cutoff $\tau(\bar p)$ as the smallest $\mu$ such that $H(\mu)\ge \bar p$. Then $p_{\mathrm{wf}}(x,s;r)\le \bar p$ if and only if $\mu(x,s,r)\le \tau(\bar p)$. Under the Poisson benchmark $H(\mu)=1-\exp(-\mu)$, $\tau(\bar p)=-\ln(1-\bar p)$; under the $k$-move variant in Appendix \ref{app:wf_microfoundation}, the relevant cutoff is $\tau_k(\bar p)$.

Proposition \ref{prop:alignment_surface_scale} therefore does not depend on Poisson arrivals or on the linear $xs$ index.
What matters is that within-form success is governed by a monotone intensity index and that the standardized component of that index increases with adoption.
Appendix Figure \ref{fig:appendix_search_surface_threshold} provides a geometric illustration of the threshold condition for when the alignment surface becomes exploitable.

\section{Illustration: Endogenous Scale and Safeguards}
\label{app:illustration}
This final section is just an illustration rather than a core proof. 
It gives a simple joint-choice illustration in which the designer endogenously chooses scale and safeguards, with modernization pressure pushing scale upward.
The aim is to show that the main threshold logic can emerge from an explicit choice problem.
The illustration is static and reduced-form: it treats $(x,r)$ as jointly chosen under a common objective rather than as a fully specified sequential game.

For tractability, we collapse safeguards to a scalar index $r\in\mathbb{R}_+$ that shifts $\eta(r)$ and enters the institutional cost $B(r)$. Fix codification at $s=1$. Consider an objective of the form
\[
U(x,r)=\lambda x-\frac{1}{2}x^2-\chi r x-\delta\Big(\alpha x-\gamma r+\eta(r)x\Big)-B(r),
\]
where $\lambda>0$ is modernization pressure, $\chi\ge 0$ captures the idea that safeguarding a larger system creates capacity frictions, $\delta\in(0,1)$ is turnover risk, and $\alpha,\gamma>0$ summarize how scale and safeguards enter a local failure-loss index. The term $\eta(r)x$ captures the within-form vulnerability index from the main text, and $B(r)$ is a convex institutional cost of safeguards. This specification can be motivated as a local approximation when the within-form success probability is not yet near one, so that $p_{\mathrm{wf}}(x;r)=1-\exp(-\eta(r)x)$ is approximately $\eta(r)x$.

For fixed safeguards $r$, the first-order condition for an interior scale choice is
\[
\frac{\partial U}{\partial x}
=
\lambda-x-\chi r-\delta\alpha-\delta\eta(r)=0,
\]
so the interior scale candidate is
\[
x^{u}(r)=\lambda-\chi r-\delta\alpha-\delta\eta(r).
\]
This expression makes the basic comparative static transparent: higher modernization pressure increases preferred scale one-for-one, while safeguards reduce preferred scale through both capacity frictions and reduced within-form vulnerability.

For an interior safeguard choice, the first-order condition is
\[
\frac{\partial U}{\partial r}
=
-\chi x+\delta\gamma-\delta\eta'(r)x-B'(r)=0,
\]
or equivalently,
\[
B'(r)+\chi x=\delta\gamma+\delta x\big(-\eta'(r)\big).
\]
The right-hand side is the marginal benefit of safeguards, which rises with scale because the within-form channel scales with $x$. The left-hand side is the marginal cost, including both institutional costs and capacity frictions.

If safeguards are slow-moving relative to adoption choices, then over the relevant horizon $r$ may adjust little even as modernization pressure increases. Once safeguards are approximately pinned, further increases in $\lambda$ raise $x$ and therefore raise the standardized vulnerability term $\eta(r)x$. This reproduces the threshold-crossing logic from the main text in an explicit choice problem: scale continues to expand with modernization pressure, while safeguards do not strengthen commensurately, so within-form vulnerability rises along the path.

A convenient parameterization for numerical examples is to let within-form intensity decay exponentially with safeguards,
\[
\eta(r)=\bar\eta e^{-\rho r},
\qquad \bar\eta>0,\;\rho>0,
\]
and to let safeguard costs be quadratic, $B(r)=\tfrac12 k r^2$ with $k>0$. 
These forms preserve $\eta'(r)<0$ and convex costs, and they make it easy to visualize the exploitability threshold \eqref{eq:alignment_surface_condition} and the critical scale \eqref{eq:xcrit_scale} in a scalar-safeguard version of the choice problem above.

\end{document}